\pdfoutput=1
\documentclass{article}
\usepackage{iclr2027_conference,times}

\usepackage{amsmath,amssymb,amsfonts,amsthm}
\usepackage{algorithm}
\usepackage{algorithmic}
\usepackage{graphicx}
\usepackage{booktabs}
\usepackage{subcaption}
\usepackage{array}
\usepackage{pifont}
\usepackage{hyperref}
\usepackage{url}
\usepackage{cleveref}
\usepackage{placeins}
\hypersetup{pdftitle={FlashSinkhorn 2: Block-Sparse Entropic Optimal Transport},
  pdfauthor={Felix X.-F. Ye, Yu Chin Fabian Lim, Naigang Wang, Davis Wertheimer}}

\newcommand{\R}{\mathbb{R}}
\newcommand{\bx}{\mathbf{x}}
\newcommand{\by}{\mathbf{y}}
\newcommand{\ba}{\mathbf{a}}
\newcommand{\bb}{\mathbf{b}}
\newcommand{\bC}{\mathbf{C}}
\newcommand{\bP}{\mathbf{P}}
\newcommand{\bX}{\mathbf{X}}
\newcommand{\bY}{\mathbf{Y}}
\newcommand{\mf}{\mathbf{f}}
\newcommand{\mg}{\mathbf{g}}
\newcommand{\mm}{\mathbf{m}}
\newcommand{\bs}{\mathbf{s}}
\newcommand{\ones}{\mathbf{1}}
\newcommand{\lo}{\mathrm{lo}}
\newcommand{\hi}{\mathrm{hi}}
\newcommand{\LSE}{\mathrm{LSE}}
\newcommand{\OT}{\mathrm{OT}}
\newcommand{\eps}{\varepsilon}
\newcommand{\cO}{\mathcal{O}}
\newcommand{\bmx}[1]{#1^{\vee}}
\newcommand{\cA}{\mathcal{A}}
\newcommand{\cJ}{\mathcal{J}}
\newcommand{\bdelta}{\boldsymbol{\delta}}
\newcommand{\cmark}{\ding{51}}
\newcommand{\xmark}{\ding{55}}

\newtheorem{theorem}{Theorem}
\newtheorem{proposition}{Proposition}
\crefname{proposition}{Proposition}{Propositions}
\Crefname{proposition}{Proposition}{Propositions}

\title{FlashSinkhorn~2: Block-Sparse\\ Entropic Optimal Transport}

\author{Felix X.-F. Ye$^{1}$ \quad Yu Chin Fabian Lim$^{2}$ \quad Naigang Wang$^{2}$ \quad Davis Wertheimer$^{2}$ \\
$^{1}$Department of Mathematics \& Statistics, University at Albany, Albany, NY, USA \\
$^{2}$IBM T.\ J.\ Watson Research Center, Yorktown Heights, NY, USA \\
Correspondence to: \texttt{xye2@albany.edu}
}

\iclrfinalcopy


\begin{document}

\maketitle
\lhead{Preprint}

\begin{abstract}

Streaming GPU solvers for entropic optimal transport (EOT), such as FlashSinkhorn, avoid storing the dense kernel but still evaluate all $n\times m$ point pairs in every Sinkhorn iteration. We present \textbf{FlashSinkhorn~2} (FS2), a solver for squared-Euclidean cost on low-dimensional point clouds that solves large discrete EOT problems to a prescribed marginal residual on a single GPU by coupling two stages. A coarse stage solves on cell centroids, lifts the potentials to every point and, when a sampled marginal check rejects the lift, continues on the centroids, replacing most point-level updates. A block-sparse fine stage then removes the centroid error that coarse updates cannot. Its Morton-ordered blocks support screening and fused tensor-core execution, and a threshold set by the block masses bounds each omitted tile's contribution to every row and column. On synthetic benchmarks, FS2 reaches the target residual on all 32 problems and GeomLoss multiscale on 10. On one A100, FS2 solves discrete EOT between two $1.34\times10^8$-particle measures from a cosmological $N$-body simulation, at an entropic blur equal to the mean interparticle distance, to an all-particle marginal residual below 0.01 in under 2.5 hours. To our knowledge, it is the largest discrete EOT problem solved to this accuracy within hours. For reproducibility, we release an open-source implementation at \url{https://github.com/ot-triton-lab/flash-sinkhorn}.

\end{abstract}

\section{Introduction}
\label{sec:intro}

Optimal transport (OT) compares probability measures in applications ranging from generative modeling to cell dynamics and cosmological reconstruction \citep{arjovsky2017wasserstein, schiebinger2019optimal, frisch2002reconstruction}. Entropic regularization \citep{cuturi2013sinkhorn} yields a differentiable problem, entropic OT (EOT), solved by Sinkhorn iterations. Applications in computational biology, medical imaging, and cosmology now routinely produce point clouds with $10^5$--$10^8$ points per measure \citep{bunne2024optimal, feydy2019fast, levy2025large}. Solvers that store the dense $n\times m$ kernel exceed GPU memory beyond about $10^5$ points. MOSCOT \citep{klein2025moscot} couples a few hundred thousand cells per time point and suggests low-rank plans \citep{scetbon2021lowrank} for millions. \citet{levy2025large} recover initial particle positions from a cosmological dataset with $1.34\times10^8$ target particles by semi-discrete transport, and note that a Sinkhorn iteration between unstructured point clouds costs quadratic time, ``prohibiting very large-scale applications.'' Solvers that evaluate entries on the fly, GeomLoss \citep{feydy2019interpolating} through KeOps \citep{charlier2021kernel} and FlashSinkhorn \citep{flashsinkhorn2026} through one fused tiled kernel per update, avoid storing the $n\times m$ matrix but still perform $\cO(nm)$ work per update.

At moderate regularization, much of this work is wasted, because the entropic plan concentrates on nearby pairs. Multiscale solvers first solve a coarser problem, which cuts the number of fine iterations, and then skip distant pairs whose contribution is bounded \citep{schmitzer2019stabilized} or estimated from the coarse potentials \citep{feydy2019fast}. GeomLoss multiscale combines both on the GPU but does not test the marginals it returns. Bringing a large problem to a prescribed residual with both ideas raises three practical issues. First, the solution of the coarser problem is only approximate, and checking whether it is accurate enough costs $\cO(nm)$. Second, screening all point pairs to find the useful ones costs $\cO(nm)$ as well. Third, the surviving pairs are distributed unevenly across rows, yet GPU tensor cores require fixed-size tiles.

We introduce \textbf{FlashSinkhorn~2} (FS2), which addresses these issues for squared-Euclidean costs in $d\le3$ by coupling a coarse stage with a block-sparse fine stage. FS2 first groups nearby points into cells and solves the much smaller problem between the cell centroids. It then lifts this solution to the original points and checks it on randomly drawn rows and columns instead of all $nm$ pairs. If the check fails, FS2 refines the centroid solution further and, if the check still fails, removes the error of replacing points by centroids with sparse updates on the points. These updates skip negligible pairs by screening fixed-size blocks of nearby points instead of single points, and each block pair they keep maps onto whole tensor-core tiles. Every update runs FlashSinkhorn's streaming kernel, so FS2 stores block-pair indices but no kernel entries. Our contributions are:

\begin{itemize}
  \item \textbf{Error guarantees.} Building on \citet{schmitzer2019stabilized}, we bound each omitted tile's contribution to every row and column, and hence the full marginal residual. For the lift, we isolate the centroid error that only fine updates remove.

  \item \textbf{GPU co-design.} Cells and blocks are contiguous ranges of one Morton order. The fine stage screens block pairs, cutting the cost to $\cO(nm/B^2)$, and runs each sparse half-step through FlashSinkhorn's fused kernel on the kept tiles.

  \item \textbf{Large-scale evaluation.} On one A100, FS2 solves discrete EOT between two $1.34\times10^8$-particle measures to a marginal residual below 0.01 in under 2.5 hours. To our knowledge, it is the largest discrete EOT problem solved to this accuracy within hours.

\end{itemize}

\section{Background and Related Work}
\label{sec:background}

\subsection{Entropic OT and FlashSinkhorn}
\label{sec:flashsinkhorn}

For point clouds $\bX=(x_i)_{i=1}^n$, $\bY=(y_j)_{j=1}^m$ in $\R^d$, strictly positive probability weights $\ba\in\Delta^n$, $\bb\in\Delta^m$, and squared Euclidean cost $C_{ij}=\|x_i-y_j\|^2$, EOT solves
\begin{equation}
\OT_\eps(\ba,\bb)=\min_{\bP\in\Pi(\ba,\bb)}
\langle\bP,\bC\rangle+\eps\,\mathrm{KL}(\bP\,\|\,\ba\otimes\bb),
\qquad \eps>0,
\label{eq:eot}
\end{equation}
where $\Pi(\ba,\bb)=\{\bP\ge0:\bP\ones=\ba,\bP^\top\ones=\bb\}$. The optimal plan is $P_{ij}=a_i b_j\exp((f_i+g_j-C_{ij})/\eps)$, with dual potentials $\mf\in\R^n$, $\mg\in\R^m$. With $\LSE_j(z_j)=\log\sum_j\exp(z_j)$ denoting the log-sum-exp over the index $j$, log-domain Sinkhorn uses the half-steps
\begin{equation}
\begin{aligned}
f_i\leftarrow-\eps\,\LSE_j\!\left[(g_j-C_{ij})/\eps+\log b_j\right],\qquad
g_j\leftarrow-\eps\,\LSE_i\!\left[(f_i-C_{ij})/\eps+\log a_i\right].
\end{aligned}
\label{eq:sinkhorn}
\end{equation}
Alternating updates apply the half-steps sequentially. Symmetric updates compute both from the same previous potentials, then average each updated potential with its previous value \citep{feydy2020geometric}. At converged balanced potentials, the source-coordinate gradient of EOT is $2a_i(x_i-T_\eps(x_i))$, where $T_\eps(x_i)$ is the barycentric projection of $x_i$ \citep{li2025robust}.

FlashSinkhorn (FS) rewrites each half-step as a row-wise LSE reduction of a biased dot-product score matrix \citep{flashsinkhorn2026}. For squared Euclidean cost, $C_{ij}=\|x_i\|^2+\|y_j\|^2-2x_i^\top y_j$, so with shifted potentials $\hat f_i=f_i-\|x_i\|^2$ and $\hat g_j=g_j-\|y_j\|^2$ the first half-step becomes $\hat f_i\leftarrow-\eps\,\LSE_j\!\left[ \big(2x_i^\top y_j+\hat g_j\big)/\eps+\log b_j\right]$, and the second exchanges the points and weights. A fused Triton kernel \citep{tillet2019triton} computes each half-step in one pass over tiles, fixed-size submatrices of the score matrix. For each tile it forms the dot products $x_i^\top y_j$ as a matrix product (GEMM), adds the bias, and folds the tile's scores into a running LSE \citep{milakov2018online}. The matrix products run on TF32 tensor cores, which offer higher throughput than FP32 at reduced mantissa precision. Scores and running statistics stay on chip and the kernel writes only the updated potentials to high-bandwidth memory (HBM), so FlashSinkhorn uses $\cO((n+m)d)$ working memory without materializing the $n\times m$ score matrix. KeOps \citep{charlier2021kernel} evaluates the same reduction pair by pair in scalar FP32 arithmetic and forms no matrix product, so it does not use the tensor cores. Each dense half-step still evaluates all $nm$ point pairs. The reduction streams over tiles, so a block mask can select the tiles the kernel visits, as in block-sparse variants of FlashAttention \citep{dao2022flashattention,dong2024flexattention}.

\subsection{Multiscale Sinkhorn and Related Approaches}
\label{sec:family}

The remaining question is which tiles to skip. Prior multiscale Sinkhorn methods answer it through spatial coarsening and potential-based truncation. In the plan, the factor $e^{-C_{ij}/\eps}$ falls to $1/e$ at $\|x_i-y_j\|=\sqrt\eps$, so $\sqrt\eps$ is the spatial scale of the regularization, called the \textit{blur}. In a locally uniform cloud, a ball whose radius is proportional to the blur holds few of the $m$ target points when the blur spans a few nearest-neighbour spacings, so most point pairs carry negligible mass and screening pays off. The number of points in such a ball grows exponentially with the dimension $d$, as the $d$-th power of the blur over the spacing, and the methods below, including ours, therefore target $d\le3$.

\citet{schmitzer2019stabilized} combines coarsening and potential-based truncation with $\eps$-annealing. It truncates the kernel at the current potentials, keeping only entries above a threshold, and bounds the resulting error in the primal--dual gap (Eq.~(3.7) and Propositions~10--11). To find these entries without testing all of them, it descends a hierarchy of cells and discards each cell pair that a bound shows to be negligible. This search combines potential maxima and cost lower bounds with entrywise refinement (Definition~13 and Algorithm~4). The reported experiments use a single CPU core (Section~5.1). GeomLoss multiscale \citep{feydy2019fast,feydy2020geometric} anneals $\eps$ on weighted cell centroids and lifts the coarse potentials to the points without checking them. Before the fine updates, it builds a cell-pair mask from the last coarse potentials, keeping a pair of cells when their centroid cost is below the sum of their potentials plus a default margin of $5\eps$. KeOps \citep{charlier2021kernel} then evaluates the fine updates on the kept pairs. Because the centroid cost does not bound the point-pair costs from below, this mask approximates Schmitzer's truncation but carries no bound on the mass it drops. Domain-decomposition methods solve local entropic subproblems in parallel, on CPUs \citep{bonafini2021domain} or, for regular grids, on GPUs \citep{medina2024domain}.

\subsection{Additional Related Work}
\label{sec:related}

Other approaches replace the Sinkhorn half-steps with gradient-based or Newton-type optimization to reduce iteration counts. \citet{kemertas2023mirror} combine mirror descent with conjugate gradients, and \citet{kemertas2025truncated} achieve locally quadratic convergence with a truncated Newton method. Like Sinkhorn-Newton-Sparse \citep{tang2024accelerating} and cuRegOT \citep{qiu2026curegot}, these solvers retain $\cO(nm)$ work per iteration and become impractical at the problem sizes we target. Over-relaxed Sinkhorn iterations \citep{thibault2021overrelaxed,lehmann2022note} reduce the iteration count but also retain $\cO(nm)$ work per iteration. Nystr\"{o}m \citep{altschuler2019massively} and low-rank \citep{scetbon2021lowrank} methods approximate the kernel or plan globally rather than exploiting spatial locality. Hierarchical Refinement \citep{halmos2025hierarchical} refines low-rank couplings recursively into a bijective Monge map in log-linear time, a different output from an EOT plan at a prescribed marginal residual.

Other multiscale methods address unregularized OT \citep{schmitzer2016sparse,gerber2017multiscale}, and octree extensions apply the GeomLoss centroid test at successive resolutions \citep{schindler2025shape}. \citet{chowdhury2021quantized} scale a quantized Gromov--Wasserstein variant to more than a million points. Sparse local corrections \citep{gasteiger2021scalable} extend the Nystr\"{o}m approximation of \citet{altschuler2019massively}. Convolutional heat kernels on grids \citep{solomon2015convolutional} and nonequispaced fast Fourier transforms accelerate kernel application. \citet{lakshmanan2022nfft} report matrix-free Sinkhorn computations with $10^7$ points in low dimension. Screenkhorn fixes selected dual variables before solving the reduced problem \citep{alaya2019screenkhorn}. Native Sparse Attention \citep{yuan2025nsa} uses compressed representations to score, rank and select key/value blocks. Each selected block contains consecutive tokens, allowing contiguous loads and tiled matrix computations even when selected blocks are separated in the sequence. FS2 uses spatial blocks of points and selects tiles through the potential-dependent, mass-based rule in \Cref{sec:rule}.

\section{FlashSinkhorn 2: Algorithm and Analysis}
\label{sec:method}

FS2 rests on two observations. First, coarse updates are cheap but cannot remove the error of replacing points by centroids. FS2 therefore stays on the centroids until coarse updates stall. It switches to fine updates (Sinkhorn updates on the points themselves) only if a sampled check of the lifted potentials fails (\Cref{sec:solve}). Second, most block pairs carry negligible mass. Fine updates skip them using a threshold set by the block masses, with a bounded error (\Cref{sec:rule}). \Cref{fig:partition} illustrates the dataflow and \Cref{alg:fs2_main} summarizes the solver. Appendix~\ref{app:proofs} gives the proofs and Appendix~\ref{app:implementation} the implementation details.

\begin{figure}[t]
\centering
\includegraphics[width=\linewidth]{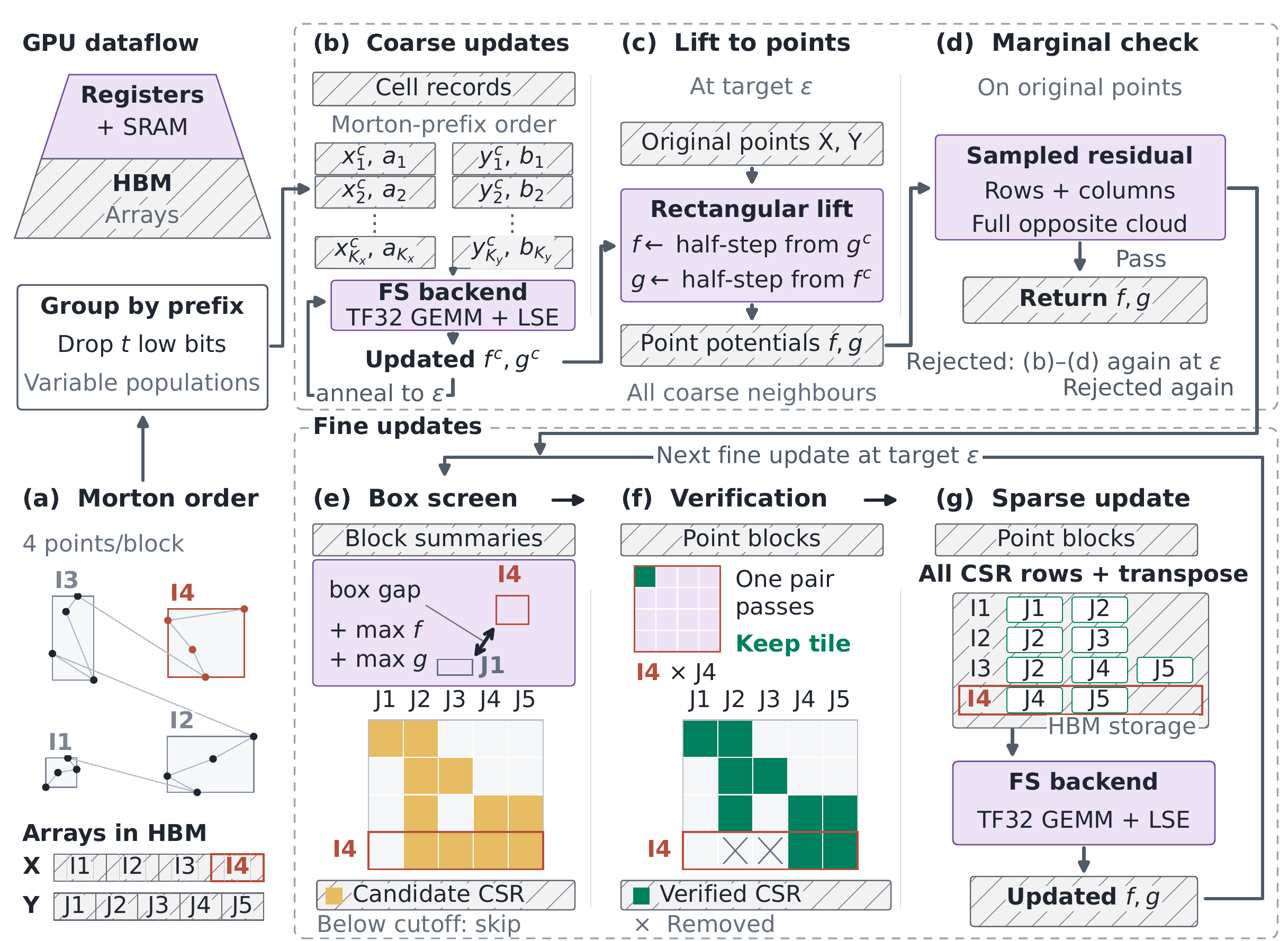}
\caption{\textbf{FlashSinkhorn~2 overview} (\Cref{alg:fs2_main}), drawn with $B{=}4$. Panel (a) is the Morton order both stages share. Panels (b)--(d) are the coarse stage (\Cref{sec:solve}) and panels (e)--(g) the fine stage (\Cref{sec:rule}). Red outlines trace source block~I4. Yellow and green mark candidate and admitted tiles. Purple denotes on-chip work, and hatched grey denotes HBM storage.}
\label{fig:partition}
\end{figure}

\begin{algorithm}[t]
\caption{FlashSinkhorn~2 forward solver.  Appendix~\ref{app:stopping-rules} gives the full stopping rules.}
\label{alg:fs2_main}
\begin{algorithmic}[1]
\REQUIRE Source $\bX$, target $\bY$, probability weights $\ba,\bb$, regularization $\eps$, tolerance $\tau$
\ENSURE pointwise potentials $(f,g)$ and the sampled-check outcome
\STATE Morton-sort, choose $B$ and $t$, form blocks, boxes, and cells $(x_I^c,a_I),(y_J^c,b_J)$
       \hfill\textit{// (a)}
\STATE anneal on centroids with FS streaming kernels to reach $\eps$
       \hfill\textit{// (b)}
\STATE lift: $f_i$, $g_j$ by one half-step of \eqref{eq:sinkhorn} against the coarse records
       \hfill\textit{// (c)}
\STATE sample $i\sim a_i$, $j\sim b_j$, then
       \textbf{if} both means plus three standard errors $<\tau$, \textbf{return}
       \hfill\textit{// (d)}
\STATE continue coarse updates at $\eps$ until they stall, lift, then
       \textbf{if} the check passes, \textbf{return}
\FOR{each fine update within budget}
    \STATE screen: keep $(I,J)$ with
           $d_{\min}^2(I,J)<\bmx{f}_I+\bmx{g}_J-\eps\,\rho_{IJ}$
           \hfill\textit{// (e)}
    \STATE verify: $\cA\leftarrow\{(I,J):m_{IJ}>\rho_{IJ}\}$ in CSR form
           \hfill\textit{// (f)}
    \STATE update: both sparse half-steps on $\cA$ from the same $(f,g)$, then average
           \hfill\textit{// (g)}
    \STATE \textbf{if} a scheduled sampled check passes, \textbf{return}
\ENDFOR
\RETURN potentials and check outcome
\end{algorithmic}
\end{algorithm}

Both stages need spatially compact groups of points stored contiguously for coalesced GPU access. We sort the points by Morton key (\Cref{fig:partition}a), which interleaves the bits of the coordinates, so points that share a key prefix are contiguous and lie in one axis-aligned box. Two partitions of the sorted order serve the two stages: cells, whose points share all but the lowest $t$ key bits, for the coarse stage of \Cref{sec:solve}, and blocks of $B$ consecutive points for the fine stage of \Cref{sec:rule}. We replace each cell by its total mass and mass-weighted centroid, giving $K_x$ source records $(x_I^c,a_I)$ and $K_y$ target records $(y_J^c,b_J)$. A finer partition reduces the distance from points to their centroids, but increases the $K_xK_y$ pairs evaluated by each coarse update. The automatic policy chooses $B\in\{64,128,256\}$ by comparing the estimated block-summary memory with the L2 budget, and takes the smallest $t$ whose coarse state fits the same budget, with at most twice as many cells as blocks.

\subsection{Coarse Sinkhorn Updates}
\label{sec:solve}

The coarse stage solves the EOT problem between the cell records. We anneal with FlashSinkhorn's streaming TF32 kernels \citep{flashsinkhorn2026} (\Cref{fig:partition}b), running coarse updates along a decreasing $\eps$-schedule, each from the potentials of the previous one. When the schedule first reaches the requested $\eps$, we \emph{lift} the coarse potentials to the original points (\Cref{fig:partition}c). The lifted $f_i$ is the first half-step of \eqref{eq:sinkhorn} with the coarse target records $(y_J^c,b_J,g_J^c)$ in place of $(y_j,b_j,g_j)$, and $g_j$ is the second with $(x_I^c,a_I,f_I^c)$ in place of $(x_i,a_i,f_i)$. The two lifts evaluate $nK_y+mK_x$ pairs.

The solver returns once a check finds both marginal residuals of the current plan, $\|\bP\ones-\ba\|_1$ and $\|\bP^\top\ones-\bb\|_1$, below a tolerance $\tau$. Computing them exactly costs $\cO(nm)$, so we check by sampling (\Cref{fig:partition}d). We draw rows $i$ with probability $a_i$ and average $|r_i/a_i-1|$, where the row sum $r_i=(\bP\ones)_i$ costs $\cO(md)$. The average is an unbiased estimate of the row residual, and columns are sampled in the same way. The check passes when both estimates plus three standard errors are below $\tau$. Every residual we report is evaluated at the returned potentials, independently of the check. A lifted plan can miss $\tau$ for two reasons, which the following bound separates.

\begin{proposition}
\label{prop:coarse-lift-residual}
Let $\bP^c$ be the entropic plan induced by any finite coarse potentials, $\lambda_c=\max_J|\log(\sum_I P^c_{IJ}/b_J)|$ its largest log column-mass error, $R_Y=\max_J\max_{j\in J}\|y_j-y_J^c\|$ the largest target-cell radius, and $\theta=2\operatorname{diam}(\bX)R_Y/\eps$. The plan $\bP$ induced by the two lifts satisfies
\begin{equation}
\|\bP\ones-\ba\|_1
\le (e^{\lambda_c}-1)+e^{\lambda_c}(e^\theta-1-\theta).
\label{eq:coarse-lift-bound}
\end{equation}
With the clouds exchanged, the same bound holds for the column residual $\|\bP^\top\ones-\bb\|_1$.
\end{proposition}

The proof is in Appendix~\ref{app:coarse-lift-bound}. The first term comes from inexact coarse marginals and vanishes at coarse convergence. The second bounds the error of replacing points by centroids. It grows like $\theta^2/2$ for small $\theta$, and at coarse convergence it reduces to $e^\theta-1-\theta$, which no coarse update changes. When the check rejects a lift, the solver therefore shrinks the first term with further coarse updates at the target $\eps$, which cost only $K_xK_y$ centroid pairs each, until its potentials stall or 2048 extra updates have run, and then lifts and checks again. If this check also fails, the solver turns to fine updates on the points, repeats the check on a schedule, and reports that the tolerance was not confirmed if the fine-update budget runs out first. On all our problems the bound of \Cref{prop:coarse-lift-residual} exceeds $\tau$, so the sampled check, not the bound, decides whether to accept a lift.

Outside the forward solve, the coarse records also yield a surrogate gradient for optimization.\label{sec:coarse-direction} Holding the coarse target records $(y_J^c,b_J,g_J^c)$ fixed makes each lifted $f_i$ a function of $x_i$ alone, and the surrogate $\widetilde{\mathcal L}=\sum_i a_if_i$ has gradient $\widetilde G_i=2a_i(x_i-\widetilde T_i)$, where $\widetilde T_i$ is the barycentric projection of $x_i$ against these records. The gradient costs the lift's $nK_y$ pairs and needs no fine update. Appendix~\ref{app:coarse-gradient} gives a condition on direction and step size sufficient to decrease the original objective.

\subsection{Sparse Sinkhorn Updates}
\label{sec:rule}

A dense fine update evaluates all $nm$ point pairs, so FS2 restricts each fine update to the block pairs in a mask. We build the mask over the $N_B=\lceil n/B\rceil$ source and $M_B=\lceil m/B\rceil$ target blocks of at most $B$ points by potential-dependent truncation of the current entropic plan $P_{ij}=a_i b_j e^{(f_i+g_j-C_{ij})/\eps}$. From here on, $I$ and $J$ index blocks rather than cells. Both half-steps share the mask, so its threshold must account for both block masses. The following bound makes this precise.
\begin{theorem}
\label{thm:rule}
Fix the potentials and $\delta_{\mathrm{rel}}>0$. Write $a_I=\sum_{i\in I}a_i$, $b_J=\sum_{j\in J}b_j$, and define
\begin{equation}
m_{IJ}=\max_{i\in I,j\in J}\frac{f_i+g_j-C_{ij}}{\eps},
\qquad
\rho_{IJ}=\log\frac{\delta_{\mathrm{rel}}}{\max(a_I,b_J)}.
\label{eq:exacttest}
\end{equation}
If $m_{IJ}\le\rho_{IJ}$, then $\sum_{j\in J}P_{ij}\le\delta_{\mathrm{rel}}a_i$ for every $i\in I$, and $\sum_{i\in I}P_{ij}\le\delta_{\mathrm{rel}}b_j$ for every $j\in J$.
\end{theorem}

The proof is in Appendix~\ref{app:proof}: bound each plan entry by $e^{m_{IJ}}$, then sum the weights in the opposite block. Testing every $m_{IJ}$ directly would evaluate all $nm$ point pairs, so we screen in two stages (\Cref{fig:partition}e--f). First, a \emph{box screen} rejects block pairs whose axis-aligned bounding boxes are far apart, using only block-level potential maxima $\bmx{f}_I=\max_{i\in I}f_i$, $\bmx{g}_J=\max_{j\in J}g_j$ and the minimum squared distance $d_{\min}^2(I,J)$:
\begin{equation}
\cA:=\{(I,J):m_{IJ}>\rho_{IJ}\}
\subseteq\{(I,J):d_{\min}^2(I,J)<
\bmx{f}_I+\bmx{g}_J-\eps\rho_{IJ}\}.
\label{eq:boxtest}
\end{equation}
This reduces the scan from $n md$ to $N_BM_Bd$. Second, a \emph{verification} pass evaluates the exact $m_{IJ}$ over all $B^2$ point pairs of each \emph{candidate} (a block pair that passes the box screen) and rejects those with $m_{IJ}\le\rho_{IJ}$. We store the \emph{admitted} tiles $\cA$ in CSR form, which lists, for each source block $I$, the target blocks $J$ with $(I,J)\in\cA$. The $f$ half-step reads this index and the $g$ half-step reads its transpose. If a block has no admitted tile, its sparse half-step would have no terms, so the solver computes it densely, against every point of the other cloud. The other blocks stay sparse.

Both the truncation principle, which rejects negligible pairs with the current potentials, and the block relaxation behind \Cref{eq:boxtest} come from \citet{schmitzer2019stabilized}, but FS2 differs from his solver in three ways. First, Schmitzer stores the truncated kernel, one entry per retained point pair, and iterates by sparse matrix--vector products. FS2 stores one index per admitted tile instead of its $B^2$ entries, because each sparse half-step runs FlashSinkhorn's streaming kernel (\Cref{sec:flashsinkhorn}) over these tiles and recomputes their scores on chip. Second, Schmitzer refines a cell hierarchy recursively down to point pairs, whereas we screen all block pairs in parallel at a single level and keep whole $B\times B$ tiles, since the kernel evaluates every pair of an admitted tile. Each kept tile spans two contiguous ranges of the Morton order, so the kernel loads both ranges with coalesced reads and forms their scores on tensor cores. Third, Schmitzer's threshold is uniform, whereas ours is set by the block masses. Because \Cref{thm:rule} certifies a mask only at the potentials used to build it, we rebuild the mask at every fine update. The mask also determines the error in each sparse update and in the gradient.

\begin{proposition}
\label{prop:mask-error}
Let $\nu_i$ be the fraction of row~$i$'s plan mass omitted by the mask, with positive retained mass. Then \textup{(i)}~the sparse and full row updates differ by $|f_i^{+,\,\cA}-f_i^+|=\eps\log\frac{1}{1-\nu_i}$, and \textup{(ii)}~the barycentric projection over the admitted columns satisfies $\|T_\eps^{\cA}(x_i)-T_\eps(x_i)\|\le\nu_i\,\operatorname{diam}(\bY)$.
\end{proposition}

The proof is in Appendix~\ref{app:sparse-update-error}. Both errors vanish with $\nu_i$, which \Cref{thm:rule} bounds. Each omitted tile contributes at most $\delta_{\mathrm{rel}}a_i$ to row $i$, and a source block omits at most $M_B$ tiles, so $\nu_i\le M_B\delta_{\mathrm{rel}}\,a_i/r_i$, where $a_i/r_i$ tends to~$1$ as the marginals converge. This bound guides the choice of $\delta_{\mathrm{rel}}$ for a tolerance $\tau$. The automatic policy takes $\delta_{\mathrm{rel}}=\tau/(4\max(N_B,M_B))$, so $\nu_i\le(\tau/4)\,a_i/r_i$ in every row and the mass the mask omits, summed over all rows or over all columns, is at most $\tau/4$. This also certifies the full residual. Let $\bP^{\cA}$ keep only the admitted tiles of $\bP$, and let $\Omega=\sum_{(I,J)\notin\cA}a_Ib_Je^{\bar m_{IJ}}$, where $\bar m_{IJ}$ is $m_{IJ}$ for a rejected candidate and $(\bmx{f}_I+\bmx{g}_J-d_{\min}^2(I,J))/\eps$ for a pair the box screen rejects. Then $\Omega$, at most $\tau/4$, bounds the omitted mass. At the potentials used to build the mask, $\|\bP\ones-\ba\|_1\le\|\bP^{\cA}\ones-\ba\|_1+\Omega$, and likewise for the columns. The sparse half-steps give the marginals of $\bP^{\cA}$ at no extra cost, since $(\bP^{\cA}\ones)_i/a_i=e^{(f_i-f_i^{+,\cA})/\eps}$, and $\Omega$ costs one reduction over values the two passes already form.

When an application needs an accurate gradient rather than the surrogate, FS2 evaluates the gradient $2a_i(x_i-T_\eps(x_i))$ on the points. It builds a mask at the returned potentials by the same rule and replaces $T_\eps(x_i)$ by $T_\eps^{\cA}(x_i)$, the barycentric projection over the admitted tiles of the source block containing $i$. A kernel computes it by streaming $\sum_jP_{ij}y_j$ alongside the row sum, and the transposed index gives the target-side gradient in the same way. \Cref{prop:mask-error}(ii) bounds the error that the mask adds, and the potentials' distance from convergence and the kernel's arithmetic add the rest.

Each fine update has three costs: the box screen ($\Theta(N_BM_Bd)$), the verification ($\Theta(B^2d)$ per candidate), and both sparse half-steps ($\Theta(B^2d\,|\cA|)$). As the blur $\sqrt\eps$ falls to a few nearest-neighbour spacings, $|\cA|$ drops well below $N_BM_B$ and sparse updates become cheaper than dense ones. Beyond the points and potentials, working memory holds only the indices of the candidate and admitted tiles. Appendices~\ref{app:io-model} and~\ref{app:operation-count} give the IO model and the complete operation count.

\section{Experiments}
\label{sec:experiments}

We evaluate FS2 on 32 synthetic problems and on the Quijote cosmological $N$-body simulation, measuring the time each solver takes to return potentials whose full two-sided marginal residual, $R=\max(\|\bP\ones-\ba\|_1,\|\bP^\top\ones-\bb\|_1)$, is at most a tolerance $\tau$. On the synthetic problems (\Cref{sec:forward}), we compare FS2 with GeomLoss multiscale and FlashSinkhorn and isolate its three ingredients: the coarse stage, the sparse fine stage and the surrogate gradient for optimization. On Quijote (\Cref{sec:quijote}), we scale FS2 from $2^{24}$ particles to the full $2^{27}$-particle box.

Every solve runs on one A100, and FS2 receives centered, TF32-rounded coordinates. Any potentials induce the plan $P_{ij}=a_ib_je^{(f_i+g_j-C_{ij})/\eps}$, which is the exact EOT solution for its own marginals $\bP\ones$ and $\bP^\top\ones$ \citep{peyre2019computational}. Its error therefore lies entirely in these marginals, and $R$ measures how far they are from $\ba$ and $\bb$.

\subsection{Benchmarks}
\label{sec:forward}

\textbf{Setup.} We benchmark 32 problems in $d{=}3$ with uniform weights, 16 on clustered clouds of 32 Gaussian clusters and 16 on heterogeneous clouds that mix a compact Gaussian, a thin needle and a 1\% uniform background. Both families span $n=m$ from $2^{19}$ to $2^{23}$ and $\eps$ from 0.001 to 0.02. This range of $\eps$ puts the blur $\sqrt\eps$ between 0.032 and 0.14, or 3--14\% of the side of the unit cube of the clustered clouds. All methods receive the same coordinates and the full squared-Euclidean cost, GeomLoss always runs its multiscale solver, and the target is $\tau=0.005$. Appendix~\ref{app:benchmarks} gives the data, the protocol, the endpoint of every run and a sweep of GeomLoss's scaling and cluster scale.

\begin{figure}[t]
  \centering
  \includegraphics[width=\textwidth]{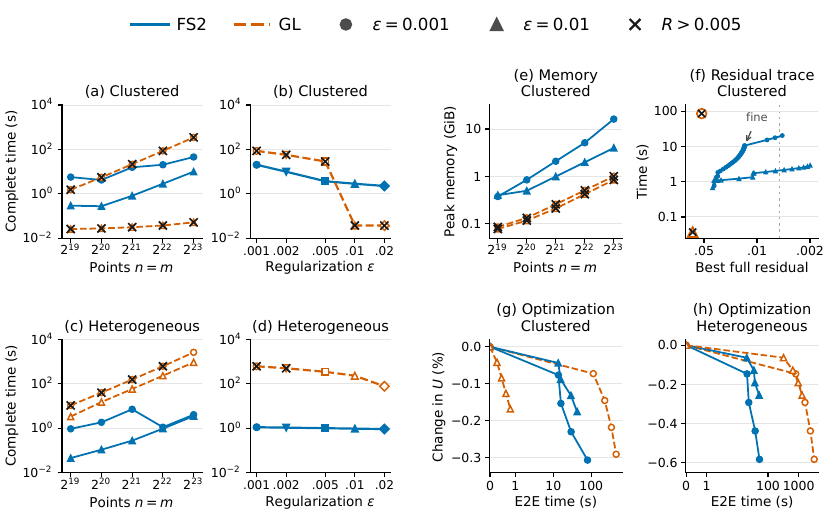}
  \caption{\textbf{FS2 vs GeomLoss (GL).} (a--d) Time of a complete solve on 26 problems: $n$ from $2^{19}$ to $2^{23}$ (a,c) and $\eps$ from 0.001 to 0.02 at $n{=}2^{22}$ (b,d). A cross marks a GL return with residual $R>\tau{=}0.005$. (e) Peak GPU memory allocated by PyTorch. (f) On clustered $n{=}2^{22}$, the time at which each new best full residual $R$ is reached, for FS2's potentials along the solve and for GL's native returns. The residual decreases to the right. The dotted line marks $\tau$ and the arrow marks the first fine update. (g,h) Change in the grouped objective $U$ over twenty optimization steps at $n{=}2^{21}$, against cumulative end-to-end (E2E) time, including compilation and autotuning.}
  \label{fig:benchmarks}
\end{figure}

\begin{table}[t]
\centering\small
\caption{FS2 against GeomLoss (GL) and FlashSinkhorn (FS) on the clustered $n{=}m{=}2^{21}$, $\eps{=}0.001$ problem. Bold marks solves that meet $\tau{=}0.005$. In (a), stage times are GPU kernel time from a separate profiled call, with sampled checks counted in the fine stage. GL$^{+}$ is GL at scaling 0.99, continued at the target $\eps$ until its full residual first meets $\tau$.}
\label{tab:clustered-solve}
\setlength{\tabcolsep}{4pt}
\begin{minipage}[t]{0.505\linewidth}
\centering\setlength{\tabcolsep}{2.5pt}
\textbf{(a)} Complete solve\\[3pt]
\begin{tabular}{@{}lrrrr@{}}
\toprule
 & FS2 & GL & GL$^{+}$ & FS \\
\midrule
Time (s) & \textbf{15.44} & 21.45 & \textbf{421.10} & 311.98 \\
\quad coarse stage & 2.94 & 0.005 & 0.022 & -- \\
\quad fine stage & 12.30 & 21.40 & 420.67 & 311.98 \\
Fine updates & 18 & 12 & 225 & 44 \\
\quad per update (s) & 0.68 & 1.78 & 1.87 & 7.09 \\
Memory (GiB) & 2.09 & 0.26 & 0.26 & 0.20 \\
Residual & \textbf{0.00483} & 0.05416 & \textbf{0.004997} & 0.05972 \\
\bottomrule
\end{tabular}

\end{minipage}\hfill
\begin{minipage}[t]{0.475\linewidth}
\centering
\textbf{(b)} One fine update per direction\\[3pt]
\begin{tabular}{@{}lrrr@{}}
\toprule
 & FS2 & GL & FS \\
\midrule
Point pairs ($10^{11}$) & 2.58 & 2.58 & 44.0 \\
Kernel time (ms) & 251 & 394 & 3{,}600 \\
ms per $10^9$ pairs & 0.97 & 1.53 & 0.82 \\
Warp instr.\ per pair & 0.28 & 0.63 & 0.34 \\
HBM read (B/$10^3$ pairs) & 14.9 & 0.32 & 4.35 \\
Tensor pipe (\%) & 21.1 & 0.07 & -- \\
Math-pipe stall (\%) & 2.3 & 17.4 & 9.0 \\
\bottomrule
\end{tabular}

\end{minipage}
\end{table}

\textbf{Speed.} FS2 meets $\tau$ on all 32 problems and GeomLoss on 10, all heterogeneous (\Cref{fig:benchmarks}a--d). At the returned potentials, the certificate of \Cref{sec:rule} confirms $R\le\tau$ for every return we evaluated on the 32 problems. Across 20 further sampling seeds on each of the seven problems closest to $\tau$, the check never accepted potentials whose residual exceeds $\tau$ (Appendix~\ref{app:certificate}). On clustered problems GeomLoss misses $\tau$ at every size and regularization. \Cref{tab:clustered-solve}a shows one such problem, $n{=}2^{21}$, in full: FS2 meets $\tau$ in 15.4\,s, GeomLoss ends its schedule after 21.45\,s at residual 0.054, and FlashSinkhorn needs 312\,s to reach 0.060, because each of its 44 updates evaluates all $n^2$ pairs. With its native schedule, GeomLoss misses $\tau$ on this problem at every scaling and cluster scale we tried, and the best setting needs 421\,s when continued at the target $\eps$ until it meets $\tau$ (GL$^{+}$). \Cref{fig:benchmarks}(f) shows how FS2 reaches $\tau$ on clustered $n{=}2^{22}$. At $\eps{=}0.001$, FS2 lifts after less than a second, because a coarse update evaluates $K_xK_y$ centroid pairs rather than $n^2$ point pairs, and the lifted residual is already below GeomLoss's final one. Further coarse updates then lower the residual, and the fine updates remove the remaining centroid error, bringing the residual below $\tau$ in about a quarter of GeomLoss's time. At $\eps{=}0.01$ the blur is $\sqrt{10}$ times wider relative to the cells, so the centroid error is smaller and the coarse stage alone reaches $\tau$. GeomLoss returns sooner but misses it.

On the ten problems where both solvers meet $\tau$, FS2 is $74$--$645\times$ faster than GeomLoss at its default cluster scale. The largest gap is on heterogeneous $n{=}2^{23}$ at $\eps{=}0.001$, where FS2 returns in 4.026\,s and GeomLoss takes 2595\,s. On all ten, FS2's first lift already passes the check, so its only work on the points is that lift and the check. GeomLoss's default gives far fewer cells, wider than the final blur, so its schedule ends with levels on the points, and it never checks its lift. FS2 picks its cell count from the input. We tested that count on one of these ten problems, where GeomLoss also stops after its coarse stage and comes close to FS2 but remains slower.

\begin{table}[t]
\small
\begin{minipage}[t]{0.58\linewidth}
\centering
\caption{Ablations on clustered $n{=}m{=}2^{20}$, $\eps{=}0.001$.}
\label{tab:components}
\setlength{\tabcolsep}{2pt}
\begin{tabular}{@{}lrrrr@{}}
\toprule
Control & Coarse & Fine & Time (s) & Time/FS2 \\
\midrule
FS2 & 2196 & 15 & 4.17 & 1 \\
No coarse stage & 0 & 130 & 153.33 & $R{>}\tau$ \\
Early coarse stop & 404 & 1115 & 215.67 & 51.59 \\
Dense fine & 2196 & 15 & 33.33 & 8.00 \\
Box only & 2196 & 15 & 4.34 & 1.04 \\
\bottomrule
\addlinespace[2pt]
\multicolumn{5}{@{}l}{\makebox[0pt][l]{\footnotesize Coarse, Fine: updates. Early coarse stop: fine cap 8192.}}
\end{tabular}

\end{minipage}\hfill
\begin{minipage}[t]{0.40\linewidth}
\centering
\caption{FS2 peak memory at $\eps{=}0.001$.}
\label{tab:memory}
\setlength{\tabcolsep}{2pt}
\begin{tabular}{@{}lrrrr@{}}
\toprule
Cloud & $n$ & $B$ & Admitted & Peak (GiB) \\
\midrule
Clustered & $2^{20}$ & 128 & 6.3\% & 0.85 \\
Heterog. & $2^{20}$ & 128 & 60.6\% & 2.72 \\
Heterog. & $2^{21}$ & 128 & 59.8\% & 9.89 \\
Clustered & $2^{23}$ & 128 & 5.2\% & 16.17 \\
Clustered & $2^{23}$ & 256 & 5.6\% & 6.88 \\
Heterog. & $2^{23}$ & 128 & no mask & 3.40 \\
\bottomrule
\end{tabular}

\end{minipage}
\end{table}

\textbf{Ablations.} Each row of \Cref{tab:components} removes or shortens one component of FS2. FS2 itself runs 2196 coarse updates and then needs just 15 fine updates. Without the coarse stage, the solve anneals on the points and misses $\tau$ even at the fine-update cap. Stopping the coarse stage after 404 updates instead leaves 1115 fine updates and makes the solve $51.6\times$ slower. With the same 15 fine updates, dense execution evaluates every tile rather than the 6.3\% the mask admits and is $8.00\times$ slower. The box screen alone, without tile verification, recovers almost all of this gain and is only 4\% slower.

\textbf{Profiling insight.} To see why FS2's fine update is fast, \Cref{tab:clustered-solve}b profiles one, in both directions, with NCU on an A100, next to a dense FlashSinkhorn update and GeomLoss's KeOps executor on FS2's mask and potentials. FS2's update evaluates only 5.9\% of the point pairs $(x_i,y_j)$ of the dense update, and NCU measures a $14\times$ shorter kernel time. Per pair, the sparse kernel is nearly as cheap as the dense one, 0.97 against 0.82\,ms per $10^9$ pairs, and executes fewer warp instructions. This holds although it gathers each admitted target block through the CSR index, while the dense kernel streams every block in order (\Cref{thm:io}). On the same mask, GeomLoss's KeOps executor takes $1.57\times$ longer and executes $2.3\times$ more instructions per pair, because FS2 computes the squared distances as tiled dot products on the tensor cores. FS2 even reads more HBM data per pair than GeomLoss, so on the same mask its speed comes from fewer instructions, not fewer bytes.

\textbf{Memory.} FS2 trades memory for speed by design. The selector keeps the coarse state and the block summaries within the L2 budget and, within it, picks the finest cells and blocks, so the coarse stage stays cheap and the mask screens at fine granularity. The price is the mask's index, 24 bytes per screened or admitted tile, which puts FS2's peak allocation at 8--10$\times$ that of either baseline in \Cref{tab:clustered-solve}a. The peak follows the mask (\Cref{tab:memory}). At $n{=}2^{20}$ the heterogeneous mask admits ten times the clustered share and needs three times the memory. At $n{=}2^{23}$ the heterogeneous solve returns its first lift and builds no mask, so its peak lies below the one at $2^{21}$. Without a mask, as at $\eps{=}0.01$ for $n{\ge}2^{20}$, the peak grows only as GeomLoss's does. At fixed $\eps$ the mask grows faster than $n$, because each point has more neighbours within the blur, so over the clustered sweep at $\eps{=}0.001$ the peak rises 40-fold while $n$ rises 16-fold (\Cref{fig:benchmarks}e). Larger blocks cut this cost. At $n{=}2^{23}$, $B=256$ instead of the selected 128 lowers the peak from 16.2 to 6.9\,GiB for slightly more time, and the selector itself takes $B=256$ for the $2^{27}$ Quijote run, whose peak stays within the 80\,GiB card.

\textbf{Optimization.}
\label{sec:optimization}
Many applications move points along an OT gradient and need a descent direction rather than an exact gradient. We test FS2's direction and its end-to-end cost by moving the source points to decrease their EOT cost $\OT_\eps$ (\Cref{eq:eot}) to a fixed target cloud, with no other term in the objective. FS2 steps along the surrogate gradient $\widetilde G$, which takes the barycentric projection over the $K_y$ coarse target centroids instead of the $m$ targets. A step therefore costs one coarse solve and one pass over the centroids, and needs no fine update. A small step along any direction whose relative error to the exact gradient is below one decreases the objective (\Cref{prop:coarse-direction}). On each of the 32 inputs, both methods take twenty normalized steps $\bX\leftarrow\bX-0.03\,H/\|H\|_F$, with $H=\widetilde G$ for FS2 and, for GeomLoss, the autograd gradient with respect to the source points, which costs about one forward solve per step. Each step moves the points by $0.03/\sqrt n$ RMS, so the twenty steps measure the direction at the initial points rather than following a long flow. Both methods' iterates are scored by the same grouped objective $U$, which fixes $K$ groups of the initial source points and of the targets and adds the EOT cost between the mass-weighted group centroids, solved accurately, to the mean squared distance of the points to their group centroids. This score ranks the iterates without a full-point solve, and the surrogate lowers it at every checkpoint on all 32 inputs. At $n{=}2^{21}$ (\Cref{fig:benchmarks}g,h), FS2 is $25$--$61\times$ faster end to end on heterogeneous data, with tied final scores. On clustered data it reaches a lower final $U$ at both $\eps$, and is also faster at $\eps{=}0.001$. GeomLoss is faster only at clustered $\eps{=}0.01$, where its forward solve takes 0.03\,s but misses $\tau$.

\subsection{Quijote Particle Transport}
\label{sec:quijote}

\textbf{Problem and data.} We test FS2 on one fiducial realization of the Quijote cosmological $N$-body suite, $512^3$ dark-matter particles in a periodic $1\,h^{-1}\mathrm{Gpc}$ box \citep{villaescusa2020quijote}. The sources are the redshift-zero positions $x_i$ and the targets are the Lagrangian lattice sites $q_i$, both with uniform weights. The particle IDs define the reference displacements but are hidden from the solve. Nested subdomains centered in the box give $n=m$ from $2^{24}$ to $2^{27}$, the full box. We fix the blur $\sqrt\eps\approx1.95\,\mathrm{Mpc}/h$ at one lattice spacing, the mean interparticle distance and 0.2--0.4\% of the domain width, and use the plain squared-Euclidean cost, which ignores the periodic wrap of the box. The evolved cloud is strongly clustered and far from the lattice. In the $2^{24}$ subdomain the median particle lies 4.6 blurs from its lattice site. At this blur the coarse stage alone does not reach the target, so every solve's fine updates start from the first lift. We set the cell level $t$ and the fine-update cap for each size. Appendix~\ref{app:quijote} gives the data, the protocol and every run.

\begin{figure}[t]
\centering
\includegraphics[width=\linewidth]{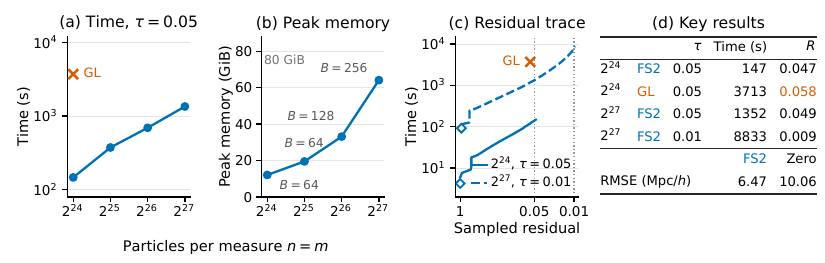}
\caption{\textbf{FS2 on Quijote.} (a) Time of a complete warm call of FS2 at $\tau{=}0.05$. The cross marks GL's instrumented call at $2^{24}$. (b) Peak GPU memory allocated by PyTorch, and the selected block size $B$. (c) Best sampled residual versus time. Diamonds mark the first fine update and dotted lines the targets. (d) Key results. Orange marks a return with residual above $\tau$, and RMSE is the displacement error of the $2^{27}$ solution at $\tau{=}0.01$ and of a zero-displacement predictor.}
\label{fig:quijote-scaling}
\end{figure}

\textbf{Scaling.} At $\tau=0.05$, FS2 meets the target at all four sizes of \Cref{fig:quijote-scaling}, verified on sampled points at the two intermediate sizes. From $2^{24}$ to $2^{27}$ particles, the time grows 9.2-fold for an 8-fold increase in $n$, and the peak memory stays within the 80\,GiB card. As on the clustered benchmarks, the fine stage does most of the work. It begins within the first 3\% of each run's time and lowers the sampled residual from about 1 to the target. GeomLoss multiscale at scaling 0.95 takes 25 times as long as FS2 at $2^{24}$ and still misses $\tau$. Its time grows almost quadratically with $n$ on the synthetic benchmarks, so we run it only at $2^{24}$. At $\tau=0.01$, a first complete call of FS2 solves the full box, $2^{54}$ particle pairs, on one A100 in under 2.5 hours. Its full marginal residual is below 0.01, even after an FP64 correction estimated on sampled points. The certificate of \Cref{sec:rule} gives the same bound from one mask build at the returned potentials, in minutes rather than the hours of the full evaluation. To our knowledge, the full box is the largest discrete EOT problem solved to this accuracy within hours, and it matches the particle count of the semi-discrete reconstruction of \citet{levy2025large}.

\textbf{Displacement reconstruction.} The plan also estimates each particle's displacement. On 4096 particles sampled from the $2^{27}$, $\tau=0.01$ solution, the barycentric projection $T_\eps(x_i)$ of \Cref{sec:flashsinkhorn}, taken over all $2^{27}$ targets, gives displacements $x_i-T_\eps(x_i)$ whose RMSE against the particle-ID displacements $x_i-q_i$, both wrapped to the periodic box, is $35.7\%$ below that of predicting no displacement (\Cref{fig:quijote-scaling}d). The three components correlate equally well with the reference, and particles whose reference trajectories cross a periodic face have larger errors than the rest.

\section{Conclusion and Limitations}
\label{sec:conclusion}

We introduced FlashSinkhorn~2 (FS2), which solves large discrete EOT problems to a prescribed marginal residual by coupling a coarse stage on cell centroids with a block-sparse fine stage that removes the centroid error. On synthetic benchmarks, FS2 reaches the target residual on all 32 problems and GeomLoss on 10. On one A100 it solves EOT between two $1.34\times10^8$-particle measures, at a blur of one lattice spacing, to a marginal residual below 0.01 in under 2.5 hours.

FS2 targets squared-Euclidean cost in $d\le3$, where a blur-sized ball holds few points and screening pays off. Reaching the $10^{10}$ points that \citet{levy2025large} set as the goal for cosmology needs blocks spread over several GPUs, because the coordinates of one measure alone take 112\,GiB in FP32, more than an A100's 80\,GiB. It also needs a sub-quadratic pair search, such as a dual-tree traversal \citep{lee2011dualtree}, because the box screen tests every block pair.

\subsection*{AI use statement}
Generative AI tools (Anthropic's Claude through Claude Code, and OpenAI's Codex) assisted with code, experiments, figures and editing. The authors checked every reported number against the experiment records and take full responsibility for the content.

\bibliography{arxiv_references}
\bibliographystyle{iclr2027_conference}

\appendix

\section{Proofs for \texorpdfstring{\Cref{sec:method}}{Section 3}}
\label{app:proofs}

\subsection{Proof of \texorpdfstring{\Cref{prop:coarse-lift-residual}}{Proposition 1}}
\label{app:coarse-lift-bound}

Let $f^c,g^c$ be finite coarse potentials, not necessarily converged, with coarse plan and target marginal ratios
\[
P^c_{IJ}=a_Ib_Je^{(f^c_I+g^c_J-C(x^c_I,y^c_J))/\eps},
\qquad
c_J=\frac{\sum_IP^c_{IJ}}{b_J}.
\]
The first lift defines weights $w_{iJ}=b_Je^{(f_i+g^c_J-C(x_i,y^c_J))/\eps}$ with $\sum_Jw_{iJ}=1$, and $\pi_{I\mid J}=P^c_{IJ}/(b_Jc_J)$ is the conditional distribution over source centroids for target cell $J$, with mean $u_J=\sum_I\pi_{I\mid J}x^c_I$. For $j\in J$ write $z_j=y_j-y^c_J$ and
\[
\phi_{iJ}(z)=\Big[\sum_I\pi_{I\mid J}e^{2(x^c_I-x_i)^\top z/\eps}\Big]^{-1}.
\]
Substituting both lifts into the fine plan and expanding the squared distances, the two copies of $e^{-\|z_j\|^2/\eps}$ cancel, and
\begin{equation}
\frac{(\bP\ones)_i}{a_i}
=\sum_J\frac{w_{iJ}}{c_J}\sum_{j\in J}\frac{b_j}{b_J}\,\phi_{iJ}(z_j).
\label{eq:lift-identity}
\end{equation}

\begin{proof}
Fix $J$ and $z=z_j$, and put $\ell_I=2(x^c_I-x_i)^\top z/\eps$ and $F(s)=(\sum_I\pi_{I\mid J}e^{s\ell_I})^{-1}$, so that $\phi_{iJ}(z)=F(1)$. Exact centroids are convex combinations of points, so $|\ell_I|\le2\operatorname{diam}(\bX)R_Y/\eps=\theta$. With $\mu_s$ and $v_s$ the mean and variance of $\ell_I$ under probabilities proportional to $\pi_{I\mid J}e^{s\ell_I}$, we have $F'(s)=-F(s)\mu_s$ and $F''(s)=F(s)(\mu_s^2-v_s)$. Since $\mu_s^2+v_s\le\theta^2$ and $F(s)\le e^{s\theta}$ on $[0,1]$, $|F''(s)|\le e^{s\theta}\theta^2$, and
\[
|F(1)-1-F'(0)|\le\int_0^1(1-s)e^{s\theta}\theta^2\,ds=e^\theta-1-\theta.
\]
The cell average of $F'(0)=-2(u_J-x_i)^\top z/\eps$ vanishes because $y^c_J$ is the exact mass-weighted centroid, so $|\sum_{j\in J}(b_j/b_J)\phi_{iJ}(z_j)-1|\le e^\theta-1-\theta$. Substituting into \Cref{eq:lift-identity} with $\sum_Jw_{iJ}=1$, $c_J^{-1}\le e^{\lambda_c}$ and $|c_J^{-1}-1|\le e^{\lambda_c}-1$ gives, for every row,
\[
\Big|\frac{(\bP\ones)_i}{a_i}-1\Big|
\le(e^{\lambda_c}-1)+e^{\lambda_c}(e^\theta-1-\theta),
\]
and multiplying by $a_i$ and summing proves \Cref{eq:coarse-lift-bound}. The column bound exchanges the roles of the two clouds.
\end{proof}

Exact coarse convergence does not remove the second term. Let both measures put mass $1/2$ at $-h$ and $h$, aggregated into one cell at zero. Then $f^c=g^c=0$ is optimal, the lifted potentials are $f_i=g_j=h^2$, and $\|\bP\ones-\ba\|_1=\|\bP^\top\ones-\bb\|_1=\cosh(2h^2/\eps)-1>0$.

\textbf{Size of the bound on the grid.} On the 32 grid problems of \Cref{app:speed}, $\theta$ ranges from 2.1 to 143 for rows and from 2.1 to 145 for columns, so the second term of \Cref{eq:coarse-lift-bound} exceeds 4.8 everywhere and the bound never accepts a lift, while the 20 lifted returns have residuals of at most 0.0041. The two terms still separate the two ways a lift fails. On clustered problems with $\eps\ge0.002$, extra coarse updates lower the lifted residual. On heterogeneous problems at $\eps=0.001$ and $n\le2^{21}$, and on clustered $n=2^{19}$ at $\eps=0.001$, it stays flat, and fine updates are required.

\textbf{Stored centroids.} The proof uses exact centroids. FS2 rounds the aggregated coordinates once to TF32, so that the same stored coordinates supply the norms, the dot products and the lift, and this rounding changes the potentials by about $1.6\times10^{-4}$ and the lift's residual by 0.5\%.

\subsection{Coordinate Gradients from Coarse Potentials (\texorpdfstring{\Cref{prop:coarse-barycentric,prop:coarse-direction}}{Propositions 3 and 4})}
\label{app:coarse-gradient}

At each outer step the optimization rebuilds the geometry at the current iterate $\bX$ and anneals on the cells until it first reaches $\eps$. The target centroids $y^c_J$, masses $b_J$ and potentials $g^c_J$ are then held fixed, so gradients do not pass through the partition, the centroids or the coarse updates. With the source coordinates replaced by $Z=(z_i)_i$, let
\[
w_{iJ}(Z)=
\frac{b_J\exp((g_J^c-\|z_i-y_J^c\|^2)/\eps)}
{\sum_L b_L\exp((g_L^c-\|z_i-y_L^c\|^2)/\eps)},
\qquad
\widetilde T_i(Z)=\sum_J w_{iJ}(Z)\,y_J^c.
\]
At $Z=\bX$ these are the lift weights of \Cref{app:coarse-lift-bound} and the projection $\widetilde T_i$ of \Cref{sec:coarse-direction}, and direct differentiation gives $\nabla_{z_i}\widetilde{\mathcal L}(Z)=2a_i(z_i-\widetilde T_i(Z))$, the gradient formula of \Cref{sec:coarse-direction}.

The surrogate gradient can differ from a gradient built on a fine barycentric projection in two ways. Let $p_{ij}$ be any normalized conditional distribution over the target points for source row $i$, $p_{iJ}=\sum_{j\in J}p_{ij}$, and, when $p_{iJ}>0$, $\bar y_{iJ}=p_{iJ}^{-1}\sum_{j\in J}p_{ij}y_j$. With $w_{iJ}=w_{iJ}(\bX)$, let $r_J=\max_{j\in J}\|y_j-y_J^c\|$ and $\operatorname{TV}_i=\frac12\sum_J|p_{iJ}-w_{iJ}|$.

\begin{proposition}
\label{prop:coarse-barycentric}
For $T_i=\sum_jp_{ij}y_j$ and $\widetilde T_i=\sum_Jw_{iJ}y_J^c$,
\begin{equation}
\|T_i-\widetilde T_i\|
\le \sum_Jp_{iJ}r_J
+\operatorname{diam}(\bY^c)\,
\operatorname{TV}_i.
\label{eq:coarse-barycentric-error}
\end{equation}
Consequently, for $G_i=2a_i(x_i-T_i)$ and $\widetilde G_i=2a_i(x_i-\widetilde T_i)$,
\begin{equation}
\|G-\widetilde G\|_F
\le 2\left[\sum_i a_i^2
\left(\sum_Jp_{iJ}r_J+
\operatorname{diam}(\bY^c)\operatorname{TV}_i
\right)^2\right]^{1/2}.
\label{eq:coarse-gradient-error}
\end{equation}
\end{proposition}

\begin{proof}
Insert the intermediate barycentre $T_i^c=\sum_Jp_{iJ}y_J^c$. The within-cell term satisfies
\[
\|T_i-T_i^c\|
=\left\|\sum_Jp_{iJ}(\bar y_{iJ}-y_J^c)\right\|
\le\sum_Jp_{iJ}r_J.
\]
The differences $p_{iJ}-w_{iJ}$ sum to zero, so their positive and negative parts both have total mass $\operatorname{TV}_i$. If $\operatorname{TV}_i=0$, the two centroid averages coincide. Otherwise
\[
T_i^c-\widetilde T_i
=\operatorname{TV}_i\left(\sum_J\frac{(p_{iJ}-w_{iJ})_+}{\operatorname{TV}_i}y_J^c
        -\sum_J\frac{(w_{iJ}-p_{iJ})_+}{\operatorname{TV}_i}y_J^c\right).
\]
The two sums in parentheses are convex combinations of the same centroid set and their distance is at most its diameter. This proves $\|T_i^c-\widetilde T_i\| \le\operatorname{diam}(\bY^c)\operatorname{TV}_i$. The first claim follows by another triangle inequality. Multiplying row $i$ by $2a_i$, squaring, and summing proves \Cref{eq:coarse-gradient-error}.
\end{proof}

The first term of \Cref{eq:coarse-barycentric-error} is quantization and the second the difference in conditional weights, which includes coarse solve error. With $p_{ij}=P^\star_{ij}/a_i$, $G$ is the gradient of the converged balanced OT value.

\begin{proposition}
\label{prop:coarse-direction}
Let $G=\nabla\mathcal L(\bX)\ne0$, let $H\ne0$ be a supplied direction, and write $c=\langle G,H\rangle/(\|G\|_F\|H\|_F)$. Suppose the gradient of $\mathcal L$ is $\beta$-Lipschitz, with $\beta>0$, on the segment from $\bX$ to $\bX^+=\bX-\eta_{\mathrm{opt}}H/\|H\|_F$. Then
\begin{equation}
\mathcal L(\bX^+)\le\mathcal L(\bX)
-\eta_{\mathrm{opt}}\|G\|_F c+\frac{\beta\eta_{\mathrm{opt}}^2}{2}.
\label{eq:inexact-descent}
\end{equation}
If $c>0$ and $0<\eta_{\mathrm{opt}}<2\|G\|_F c/\beta$, the objective decreases. A sufficient condition for $c>0$ is $\|H-G\|_F\le\delta\|G\|_F$ with $0\le\delta<1$, which implies $c\ge\sqrt{1-\delta^2}$.
\end{proposition}

\begin{proof}
Let $v=H/\|H\|_F$. The descent lemma for a $\beta$-smooth function gives
\[
\mathcal L(\bX-\eta_{\mathrm{opt}} v)
\le\mathcal L(\bX)-\eta_{\mathrm{opt}}\langle G,v\rangle
+\frac{\beta\eta_{\mathrm{opt}}^2}{2}.
\]
Since $\langle G,v\rangle=\|G\|_F c$, this is \Cref{eq:inexact-descent}, and the stated step-size condition makes its right-hand side strictly smaller than $\mathcal L(\bX)$.

For the relative-error claim, set $u=G/\|G\|_F$, $h=H/\|G\|_F$, and $r=\|h\|_F>0$. If $c=\langle u,h\rangle/r$, then $\|h-u\|_F^2=r^2+1-2rc\le\delta^2$, and therefore
\[
c\ge\frac{r^2+1-\delta^2}{2r}
=\frac12\left(r+\frac{1-\delta^2}{r}\right)
\ge\sqrt{1-\delta^2},
\]
where the last step is the arithmetic--geometric mean inequality.
\end{proof}

\subsection{Proof of \texorpdfstring{\Cref{thm:rule}}{Theorem 1}}
\label{app:proof}

\begin{proof}
Write the plan at the current potentials, with no convergence assumed, as
\begin{equation}
P_{ij}(\mf,\mg)=a_i\,b_j\exp\!\big((f_i+g_j-C_{ij})/\eps\big).
\label{eq:planapp}
\end{equation}
Every exponent in the tile is at most $m_{IJ}$, so $P_{ij}\le a_ib_je^{m_{IJ}}$ for $i\in I$, $j\in J$, and summing over the opposite block gives
\begin{equation}
\sum_{j'\in J}P_{ij'}\le a_i\,b_J\,e^{m_{IJ}},
\qquad
\sum_{i'\in I}P_{i'j}\le b_j\,a_I\,e^{m_{IJ}}.
\label{eq:tilebound}
\end{equation}
Since $e^{\rho_{IJ}}=\delta_{\mathrm{rel}}/\max(a_I,b_J)$ and both block masses are at most $\max(a_I,b_J)$, the hypothesis $m_{IJ}\le\rho_{IJ}$ bounds the two right-hand sides by $\delta_{\mathrm{rel}}a_i$ and $\delta_{\mathrm{rel}}b_j$.
\end{proof}

One bound thus serves both directions, because the row bound carries the factor $b_J/\max(a_I,b_J)$, the column bound carries $a_I/\max(a_I,b_J)$, and one of the two is always 1. The box screen is conservative, because $f_i\le\bmx{f}_I$, $g_j\le\bmx{g}_J$ and $C_{ij}\ge d^2_{\min}(I,J)$ give
\begin{equation}
m_{IJ}\le\big(\bmx{f}_I+\bmx{g}_J-d^2_{\min}(I,J)\big)/\eps,
\label{eq:worstmember}
\end{equation}
so a pair that fails \Cref{eq:boxtest} has $m_{IJ}\le\rho_{IJ}$. Both bounds use only pointwise inequalities at the current potentials and hold at every level of the schedule. The screen depends on the potentials only through $\bmx{f}_I+\bmx{g}_J$, so it is invariant under the shift $\mf\mapsto\mf+\kappa_0$, $\mg\mapsto\mg-\kappa_0$, which leaves the plan unchanged. A rule built on $\max(\bmx{f}_I,\bmx{g}_J)$ would give different masks for these equivalent potential pairs.

\textbf{Non-uniform weights.} Every benchmark uses uniform weights, for which $\rho_{IJ}$ is the constant $\log(\tau/4)$ whenever $B$ divides $n=m$. We rerun the two problems of \Cref{tab:mask-projection} with the non-uniform weights of FlashSinkhorn's public benchmark code \citep{flashsinkhorn2026}, $a_i\propto u_i+0.1$ with $u_i$ uniform on $[0,1]$, and with a rank$^{-1/2}$ power law. FS2 meets $\tau$ in all four cases, and the per-point bound of \Cref{thm:rule} holds everywhere. Against a scalar rule that uses the largest block mass for every pair, the mass-dependent rule admits 0.3--0.5\% fewer block pairs under the first weighting and 2.9--5.3\% fewer under the power law, with the same number of fine updates.

\subsection{Proof of \texorpdfstring{\Cref{prop:mask-error}}{Proposition 2}}
\label{app:sparse-update-error}

\begin{proof}
\textup{(i)} Write $S_i=\sum_jb_je^{(g_j-C_{ij})/\eps}$. The common factor $a_ie^{f_i/\eps}$ cancels from the omitted mass ratio, so the restricted sum is $S_i^{\cA}=(1-\nu_i)S_i$ and
\[
f_i^{+,\,\cA}-f_i^+=-\eps\log S_i^{\cA}+\eps\log S_i=-\eps\log(1-\nu_i).
\]
The column update is identical.

\textup{(ii)} The masked projection $T_\eps^{\cA}(x_i)$ restricts both sums of $T_\eps(x_i)=\sum_jP_{ij}y_j/\sum_jP_{ij}$ to the admitted columns. Subtracting it from the full projection gives the exact identity
\begin{equation}
T_\eps(x_i)-T_\eps^{\cA}(x_i)
=\frac{\sum_{J:(I,J)\notin\cA}\sum_{j\in J}
P_{ij}\big(y_j-T_\eps^{\cA}(x_i)\big)}{\sum_j P_{ij}},
\label{eq:mask-displacement}
\end{equation}
whose denominator sums over all targets, and the triangle inequality gives
\begin{equation}
\|T_\eps^{\cA}(x_i)-T_\eps(x_i)\|
\le\frac{\sum_{J:(I,J)\notin\cA}\sum_{j\in J}
P_{ij}\|y_j-T_\eps^{\cA}(x_i)\|}{\sum_j P_{ij}}
\le\nu_i\,\mathrm{diam}(\bY).
\label{eq:mask-displacement-bound}
\end{equation}
\end{proof}

The middle term of \Cref{eq:mask-displacement-bound} keeps the distance of each omitted contribution, so it is much smaller than the diameter bound when the omitted mass lies near the masked projection. The diameter bound is nevertheless sharp. Take two identical one-dimensional clouds, each with $B$ points at 0 and $B$ at $D>0$, uniform weights and one block per location. At the balanced optimum every row sends the fraction $\nu_i=(1+e^{D^2/\eps})^{-1}$ to the opposite location, so when $\nu_i\le\delta_{\mathrm{rel}}<1-\nu_i$ the admission rule keeps only the two same-location tiles. For $x_i=0$, $T^{\cA}_\eps(x_i)=0$ while $T_\eps(x_i)=\nu_iD$. By \textup{(ii)}, the source gradient $2a_i(x_i-T_\eps(x_i))$ changes by at most $2a_i\nu_i\operatorname{diam}(\bY)$ when the masked projection is used.

\textbf{Measured error.} \Cref{tab:mask-projection} evaluates \Cref{prop:mask-error} at the potentials FS2 returns on the clustered and heterogeneous $n=m=2^{20}$, $\eps=0.001$ problems, with the production $\delta_{\mathrm{rel}}$ and with $64\delta_{\mathrm{rel}}$ as a positive control. One FP64 pass over all $n\times m$ pairs splits each row's mass and first moment into admitted and omitted sums. No row exceeds either bound, and the looser mask raises the largest error 50--70-fold, so the production values are not a floor of the measurement. The block-sparse gradient kernel, run on the production mask, adds its own arithmetic error. Against the FP64 masked gradient it is within $1.0\times10^{-4}$ (clustered) and $1.0\times10^{-2}$ (heterogeneous) in relative $\ell_2$ in FP32, far above the masking error of the table.

\begin{table}[t]
\centering\footnotesize
\setlength{\tabcolsep}{3pt}
\caption{\Cref{prop:mask-error} at FS2's returned potentials, $n=m=2^{20}$, $\eps=0.001$: largest values over all rows. The $\nu_i$ column divides by the \Cref{thm:rule} bound $M_B\delta_{\mathrm{rel}}a_i/r_i$, the error column by $\nu_i\operatorname{diam}(\bY)$, and the gradient compares $2a_i(x_i-T^{\cA}_\eps(x_i))$ with $2a_i(x_i-T_\eps(x_i))$.}
\label{tab:mask-projection}
\begin{tabular}{@{}lrrrrrrr@{}}
\toprule
 & $\delta_{\mathrm{rel}}$ & & & $\nu_i$ over & & $\|T^{\cA}_\eps-T_\eps\|$ & Gradient \\
Cloud & factor & Admitted & Max $\nu_i$ & Thm.~1 bound & Max error & over $\nu_i\operatorname{diam}(\bY)$ & rel.\ $\ell_2$ \\
\midrule
Clustered & $1$ & 6.33\% & $6.7\times10^{-6}$ & 0.005 & $2.7\times10^{-6}$ & 0.48 & $1.6\times10^{-7}$ \\
Clustered & $64$ & 5.32\% & $3.7\times10^{-4}$ & 0.005 & $1.4\times10^{-4}$ & 0.48 & $1.3\times10^{-5}$ \\
Heterog. & $1$ & 60.54\% & $2.9\times10^{-5}$ & 0.023 & $2.4\times10^{-6}$ & 0.23 & $9.1\times10^{-6}$ \\
Heterog. & $64$ & 49.24\% & $3.0\times10^{-3}$ & 0.038 & $1.7\times10^{-4}$ & 0.23 & $1.8\times10^{-3}$ \\
\bottomrule
\end{tabular}

\end{table}

\subsection{IO Model (\texorpdfstring{\Cref{thm:io}}{Proposition 5})}
\label{app:io-model}

The dense FlashSinkhorn kernel moves $\Theta(nd+md+nmd^2/M)$ scalars between HBM and fast memory of size $M$ \citep{flashsinkhorn2026}. For the sparse traversal of \Cref{alg:fs2_lse}, the count is governed by neighbour-record misses.

\begin{proposition}
\label{thm:io}
Let a neighbour block record occupy $\sigma=\Theta(B(d+1))$ scalars. Suppose one source block's queries and accumulators stay resident in fast memory of size $M\ge c_0\sigma$ for a fixed constant $c_0$, costs are computed online, CSR indices are read once per traversal, and caches are cold at the start of each half-step. If $L_f$ and $L_g$ count neighbour-record misses along the CSR and its transpose, one iteration evaluating both directions moves
\begin{equation}
Q_{\mathrm{iter}}=
\Theta\!\big((n+m)(d+1)+|\cA|+\sigma[L_f+L_g]\big)
\label{eq:io}
\end{equation}
scalars between slow and fast memory. Each miss count lies between the number of active neighbour blocks in its traversal and $|\cA|$.
\end{proposition}

\begin{proof}
We count the $\hat\mf$ half-step, and the $\hat\mg$ half-step is the same count on the transposed index. Each of the $N_B$ source blocks moves its queries $\bX_I$ and writes back $\hat\mf_I$ once, $\Theta(B(d+1))$ scalars, because both stay resident for the whole row. Over both directions this is $\Theta((n+m)(d+1))$. The traversal reads $\mathrm{crow}$ once and each entry of $\mathrm{col}$ once, $\Theta(N_B+|\cA|)$ scalars, and $N_B\le n$ is absorbed by the first term. An empty row reads no neighbour record. Its repair is a dense reduction, counted with the dense updates in \Cref{eq:forward-operation-count}. Visiting an admitted pair needs $\bY_J$, $\hat\mg_J$ and $\bdelta_J$, $\sigma=\Theta(B(d+1))$ scalars, and moves nothing when the record is already resident. The $L_f$ visits that miss contribute $\Theta(\sigma L_f)$. Summing gives \Cref{eq:io}. Each half-step starts cold, so every active neighbour block misses at least once. At most one miss is charged per admitted pair, so $L_f$ lies between the number of active neighbour blocks and $|\cA|$, and so does $L_g$.
\end{proof}

On-chip memory is private to each program, and within one row the CSR entries are distinct, so a neighbour record can be reused across source blocks only through L2. The upper end, one gather per admitted pair, is therefore the bound that always holds. Unlike the dense case, enlarging $M$ does not reduce it, because the admitted columns of a row form an arbitrary subset.

\section{Implementation Details}
\label{app:implementation}

\subsection{Block Size and Coarse Cells}
\label{app:block-size}

The automatic policy chooses the block size from $B\in\{64,128,256\}$ by estimating the memory of the block summaries. Each summary holds a bounding box, a potential maximum and a log block mass, 32 bytes for $d=3$, so the summaries of both clouds take $32(N_B+M_B)$ bytes, with $N_B=\lceil n/B\rceil$ and $M_B=\lceil m/B\rceil$. The chosen $B$ keeps this estimate within $0.8C_{\mathrm{L2}}$, where $C_{\mathrm{L2}}$ is the physical L2 cache size, so the budget forces larger blocks on larger clouds.

For each $t\in\{0,\ldots,62\}$, we count the occupied prefixes of the sorted 63-bit Morton keys. Smaller $t$ gives a finer partition. We choose the smallest $t$ satisfying
\begin{equation}
\begin{gathered}
K_x\le2N_B,\qquad K_y\le2M_B,\qquad
112(K_x+K_y)\le0.8C_{\mathrm{L2}},\\
16\max(K_x,K_y,BN_B,BM_B)\le2^{31}-1.
\end{gathered}
\label{eq:automatic-coarse-selection}
\end{equation}
The first two conditions limit the coarse point count relative to the execution blocks. The third limits estimated storage for coarse arrays, and the fourth checks signed 32-bit offsets for coordinates padded to width 16. The fourth errs by one on the safe side, so it excludes $n=2^{27}$, whose largest offset $16n-1=2^{31}-1$ still fits. Cell counts come from the sorted keys, without solving a coarse problem for each $t$. The budget constants are empirical. They make the cells and blocks as fine as memory allows, which raises memory traffic.

\subsection{Forward Solver and Stopping Rules}
\label{app:stopping-rules}

\textbf{Sampled check.} In exact arithmetic $(\bP\ones)_i/a_i=\exp((f_i-f_i^+)/\eps)$, where $f_i^+$ is the unrestricted row update, so for potentials fixed independently of the sample, drawing $i$ with probability $a_i$ gives
\begin{equation}
\mathbb E_{i\sim\ba}\left|\exp((f_i-f_i^+)/\eps)-1\right|
=\sum_ia_i\left|(\bP\ones)_i/a_i-1\right|=\|\bP\ones-\ba\|_1,
\label{eq:coarse-screen}
\end{equation}
and columns drawn with $j\sim\bb$ estimate $\|\bP^\top\ones-\bb\|_1$ in the same way. The check draws 1024 indices per side with seed 0, reduces each sampled marginal over all opposite points, and bounds each residual by its sample mean plus three standard errors. It passes when the larger bound is below $\tau$ and fails when a mean minus three standard errors exceeds $\tau$. Otherwise it evaluates 8192 further indices and passes only if their bound is below $\tau$. No sample size gives a distribution-free guarantee on a mean \citep{bahadur1956nonexistence}. With uniform weights, $S$ draws miss a given row with probability $(1-1/n)^S$, and the solver draws the indices once per solve and reuses them at every check, so the potentials at the stopping check depend on the sample, whereas \Cref{eq:coarse-screen} assumes they do not. The check can therefore accept potentials whose full residual exceeds $\tau$, and every reported residual comes from a full evaluation, except at two intermediate Quijote sizes. \Cref{app:certificate} measures how often the check accepts such potentials on the grid.

\textbf{Coarse stage.} The coarse solve anneals on the Morton-prefix cells of \Cref{sec:method} with GeomLoss's $\eps$-scaling schedule, from the diagonal of the joint bounding box down to the blur $\sqrt\eps$ with scaling 0.95 and averaged updates, and lifts when it first reaches the requested $\eps$. If the sampled check rejects the lift, the solver completes the 64-entry target-$\eps$ tail, lifts and checks again, and then continues in batches of 64 coarse updates. After each batch it evaluates, from the shifted coarse potentials before and after the last averaged update,
\begin{equation}
D_c=\frac{2}{\eps}\max\Big\{\sum_I a_I\big|\hat f^{c+}_I-\hat f^c_I\big|,\
\sum_J b_J\big|\hat g^{c+}_J-\hat g^c_J\big|\Big\},
\label{eq:dc}
\end{equation}
accumulated in FP64. In exact arithmetic $D_c$ is the larger of the two mass-weighted coarse log-marginal errors before that update. These are means, whereas $\lambda_c$ in \Cref{prop:coarse-lift-residual} is a maximum over target cells, so the stop does not bound the first term of \Cref{eq:coarse-lift-bound}. Extra coarse work stops after two consecutive batches with $D_c\le\tau/2$, three consecutive relative improvements below 1\%, or 2048 extra updates. A non-finite $D_c$ discards the extra batches. The refined state is lifted and checked.

\textbf{Fine stage.} The lifted pair initializes the fine updates without another annealing schedule. The two symmetric half-steps start from the same incoming potentials, so one mask, rebuilt before every fine update, serves both. The first two ordinary updates are checked one by one. Afterwards the check interval is the smallest of 1, 2, 4, 8 and 16 updates whose combined time is at least ten times that of a check, and 1 once the 1024-index bound is at most $2\tau$, with further checks at every 16th update and at the cap. A terminal preview is a complete fine update with its own mask that forms unaveraged potentials from the same incoming pair. It runs after the re-lift that ends the coarse continuation, after the second ordinary update, at every 16th ordinary update once the 1024-index bound is at most $2\tau$, and at the cap. A rejected preview does not enter the ordinary trajectory, and the fine-update counts of \Cref{sec:forward} include previews.

\textbf{Return.} The solver returns as soon as a sampled check passes. At the fine-update cap, 128 ordinary updates unless stated otherwise, it returns the checked potentials with the best finite sample score and reports the tolerance as not confirmed. If no score is finite, it returns the retained lift.

\subsection{Mask Construction and Sparse Half-Steps}
\label{app:mask-construction}

\Cref{alg:fs2_mask} builds the candidates of the box screen \Cref{eq:boxtest} in CSR form, and verification then evaluates $m_{IJ}$ on each candidate tile. \Cref{alg:fs2_lse} gives the sparse $\hat\mf$ half-step over the admitted blocks, and the $\hat\mg$ half-step runs the same kernel on the transposed index.

\begin{algorithm}[t]
\caption{\textsc{BuildMask}: \Cref{eq:boxtest} as two Triton passes over block rows, one counting and one writing. Both scan the full $N_B\times M_B$ grid with the predicate in registers. Each tile test costs about $6d+3$ operations, about 0.010 per point pair at $d=3$ and $B=64$.}
\label{alg:fs2_mask}
\begin{algorithmic}[1]
\REQUIRE boxes $[\bx^I_{\lo},\bx^I_{\hi}]$, $[\by^J_{\lo},\by^J_{\hi}]$, block maxima $\bmx{f}, \bmx{g}$, log masses $\log a, \log b$, $\eps$, $\delta_{\mathrm{rel}}$
\ENSURE CSR arrays $\mathrm{crow}$, $\mathrm{col}$
\FOR{each block row $I$ \textbf{in parallel}}
    \STATE $c_I \leftarrow 0$ \hfill \textit{// pass 1: count}
    \FOR{each tile $\cJ$ of column blocks}
        \STATE $W \leftarrow \max\!\big(\bx^I_{\lo}-\by^{\cJ}_{\hi},\; \by^{\cJ}_{\lo}-\bx^I_{\hi},\; 0\big)$ \hfill \textit{// per-axis box separation}
        \STATE $d^2_{\min} \leftarrow \mathrm{rowsum}(W \odot W)$ \hfill \textit{// lower-bounds every pair in the tile}
        \STATE $\mathrm{keep} \leftarrow \big[\, d^2_{\min} < \bmx{f}_I + \bmx{g}_{\cJ} - \eps\,\rho_{I\cJ} \,\big]$ \hfill \textit{// \Cref{eq:boxtest}, $\log a_I$ is hoisted}
        \STATE $c_I \leftarrow c_I + \mathrm{sum}(\mathrm{keep})$
    \ENDFOR
    \STATE write $c_I$ to HBM
\ENDFOR
\STATE $\mathrm{crow} \leftarrow \mathrm{ExclusiveScan}(c)$ \hfill \textit{// offsets exist before any row writes}
\STATE \textbf{if} $\mathrm{crow}[N_B]$ exceeds the index type \textbf{then raise} \hfill \textit{// refuse rather than wrap silently}
\FOR{each block row $I$ \textbf{in parallel}}
    \STATE $\mathrm{base}\leftarrow \mathrm{crow}[I]$, \ $w\leftarrow 0$ \hfill \textit{// pass 2: fill}
    \FOR{each tile $\cJ$ of column blocks}
        \STATE recompute $\mathrm{keep}$ in registers \hfill \textit{// cheaper than storing a dense Boolean mask}
        \STATE $\mathrm{pos} \leftarrow \mathrm{ExclusiveScan}(\mathrm{keep})$ \hfill \textit{// slot inside this tile}
        \STATE $\mathrm{col}[\mathrm{base}+w+\mathrm{pos}] \leftarrow \cJ$ \ \textbf{where} $\mathrm{keep}$ \hfill \textit{// one program per row, so ascending}
        \STATE $w \leftarrow w + \mathrm{sum}(\mathrm{keep})$
    \ENDFOR
\ENDFOR
\end{algorithmic}
\end{algorithm}

\begin{algorithm}[t]
\caption{\textsc{SparseLSE}: the streaming $\hat\mf$-update over the admitted blocks. The online max/sum-exp recurrence is FlashSinkhorn's, run here over a CSR row rather than over every column block. The cost is never formed either way.}
\label{alg:fs2_lse}
\begin{algorithmic}[1]
\REQUIRE $\bX\in\R^{n\times d}$, $\bY\in\R^{m\times d}$, $\hat\mg\in\R^m$, $\bb\in\Delta^m$, CSR $(\mathrm{crow},\mathrm{col})$, $\eps$
\ENSURE $\hat\mf\in\R^n$
\STATE $\bdelta \leftarrow \eps\log\bb$ \hfill \textit{// precompute the weight bias}
\FOR{each block row $I$ of size $B$}
    \STATE load $\bX_I$ to SRAM, \ $\mm_I \leftarrow -\infty$, \ $\bs_I \leftarrow 0$ \hfill \textit{// running max and sum-exp}
    \STATE $r_0 \leftarrow \mathrm{crow}[I]$, \ $\mathrm{nnz}_I \leftarrow \mathrm{crow}[I{+}1]-r_0$
    \FOR{$t = 0,\dots,\mathrm{nnz}_I-1$}
        \STATE $J \leftarrow \mathrm{col}[r_0+t]$, load $\bY_J, \hat\mg_J, \bdelta_J$ to SRAM \hfill \textit{// admitted blocks only}
        \STATE $\mathbf{Z} \leftarrow (2\,\bX_I \bY_J^\top + \hat\mg_J + \bdelta_J)/\eps$ \hfill \textit{// one kernel tile, the cost is never formed}
        \STATE $\mm_{\mathrm{new}} \leftarrow \max(\mm_I,\ \mathrm{rowmax}(\mathbf{Z}))$
        \STATE $\bs_I \leftarrow e^{\mm_I-\mm_{\mathrm{new}}}\odot \bs_I + \mathrm{rowsum}\!\big(e^{\mathbf{Z}-\mm_{\mathrm{new}}}\big)$, \ $\mm_I \leftarrow \mm_{\mathrm{new}}$
    \ENDFOR
    \STATE $\hat\mf_I \leftarrow -\eps\big(\mm_I + \log \bs_I\big)$, write to HBM \hfill \textit{// $-\eps\,\LSE_{\rm row}$ over the admitted blocks}
\ENDFOR
\end{algorithmic}
\end{algorithm}

All forward solves form the dot products on TF32 tensor cores with FP32 accumulation, keep the potentials in FP32, and add the bias $(\hat g_j+\eps\log b_j)/\eps$ in FP32 after the dot product. Folding the bias into the dot as an extra coordinate would round an $\cO(1/\eps)$ operand to TF32, an absolute logit error of about 5 at $\eps=10^{-4}$.

\subsection{Operation Count}
\label{app:operation-count}

Let $S_c$ count all coarse updates, including annealing and additional iterations at the requested $\eps$, and let $L_{\mathrm{lift}}$ count paired lifts. Each contributes respectively $\cO(K_xK_yd)$ and $\cO((nK_y+mK_x)d)$ operations. For sampled checks, define $Q_{\mathrm{check}}=\sum_\ell(mS_{x,\ell}+nS_{y,\ell})$, where $S_{x,\ell},S_{y,\ell}$ are the actual sample sizes in check $\ell$. Both stages of an enlarged check contribute to this sum.

For fine update $k$, let $\mathcal C_k$ be the bounding-box candidates and $\mathcal A_k$ the verified tiles. Computing potential maxima and scanning block pairs costs $\cO(n+m+N_BM_Bd)$. Verification and the two sparse half-steps cost $\cO(B^2d(|\mathcal C_k|+2|\mathcal A_k|))$. Writing and transposing the indices adds $\cO(N_B+M_B+E_k\log(2+E_k))$ as a conservative bound including sorting, where $E_k=|\mathcal C_k|+|\mathcal A_k|$. Let $Q_{\mathrm{dense}}$ count point pairs evaluated by dense updates and empty-row repairs. The complete arithmetic and index work is bounded by
\begin{equation}
\begin{aligned}
\cO\Big(& (n+m)\log(n+m)+(n+m)d\\
&+d\{S_cK_xK_y+L_{\mathrm{lift}}(nK_y+mK_x)
       +Q_{\mathrm{check}}+Q_{\mathrm{dense}}\}\\
&+\sum_k\{n+m+N_BM_Bd+B^2d(|\mathcal C_k|+2|\mathcal A_k|)
             +E_k\log(2+E_k)\}\Big).
\end{aligned}
\label{eq:forward-operation-count}
\end{equation}
The sum runs over ordinary updates and terminal previews, rejected ones included, and each dense update adds its $nm$ evaluations to $Q_{\mathrm{dense}}$. No score matrix is stored. Beyond the point and coarse arrays, memory holds the candidate and admitted CSR indices, and an index with $z$ nonzeros is allocated only if it fits signed 32-bit offsets and $24z+8(N_B+M_B+2)\le0.8M_{\mathrm{free}}$ bytes, with $M_{\mathrm{free}}$ the free GPU memory. The selector's limits give $K_xK_y\le4N_BM_B$ and $nK_y+mK_x\le2(nM_B+mN_B)$. Both remain quadratic in $n$ at fixed $B$.

\FloatBarrier
\section{Synthetic Benchmarks}
\label{app:benchmarks}

After the setup and the GeomLoss baseline, the subsections follow the paragraphs of \Cref{sec:forward} in order, each with its own protocol.

\subsection{Setup}

Unless stated otherwise, every run uses one A100 with PyTorch 2.8.0+cu128, Triton 3.4.0, GeomLoss 0.3.1 and PyKeOps 2.3. The kernel profiles of \Cref{app:profiling} use a separate A100. Inputs are three-dimensional with uniform weights and seed zero, and all methods receive the same centered, TF32-rounded coordinates and the full squared-Euclidean cost, without a factor of $1/2$, so GeomLoss's default blur of 0.05 is the grid point $\eps=0.005$. Times are warmed complete calls, from preparation and parameter selection to return, and exclude input loading, compilation and validation. Grid endpoints are single calls, and comparisons and ablations report medians of three. Accuracy is the full residual $R$ of the returned potentials (\Cref{sec:experiments}), evaluated on all points, against $\tau=0.005$. Peak memory is the largest PyTorch allocation, inputs included, in a separate warmed call.

\textbf{Methods compared.}
\begin{itemize}
\item \emph{FS2}: the automatic policy of \Cref{app:implementation}, which chooses $B=128$ over 64 on every problem and, for $n=2^{19},\dots,2^{23}$, the cell levels $t=46$, 45, 44, 43 and 42 on clustered and 43, 42, 41, 40 and 39 on heterogeneous clouds.
\item \emph{GeomLoss multiscale}: its native multiscale solver with scaling 0.9, default cluster scale and truncation 5, with KeOps for the fine reductions.
\item \emph{FlashSinkhorn}: its native dense solver with scaling 0.9.
\end{itemize}
Neither baseline gets added refinement or a residual-based stopping rule.

\textbf{Data.} A clustered cloud draws 32 centers uniformly in $[0,1]^3$ and places each point at a uniformly chosen center plus isotropic Gaussian noise of standard deviation 0.02. Source and target draw their centers independently. A heterogeneous cloud mixes three components: 80\% of the points in a compact Gaussian at $(-0.35,-0.15,0)$ with standard deviations $(0.04,0.025,0.02)$, 19\% in a thin needle at $(0.35,0.10,0)$ with standard deviations $(0.20,0.008,0.008)$, rotated by 0.6\,rad about the $z$-axis, and 1\% uniform background in $[-0.9,0.9]^3$. The target shifts the two Gaussian centers by $(0.025,-0.01,0.008)$ and $(-0.018,0.018,-0.006)$ and is drawn independently.

\subsection{GeomLoss Multiscale Workflow}
\label{app:geomloss-workflow}

The evaluated GeomLoss 0.3.1 point-cloud backend \citep{feydy2019fast,feydy2020geometric} uses two resolutions, the weighted centroids of cells of side $\ell_{\mathrm{cell}}$ and the original points. It follows a decreasing schedule $\eps_1\ge\cdots\ge\eps_{N_\eps}$ on the centroids, switches to the points just before the first level whose blur falls below $\ell_{\mathrm{cell}}$, and lifts the potentials there by an LSE against the coarse records. If the schedule never reaches such a blur, or reaches it only at its last level, GeomLoss lifts its final coarse potentials and takes no fine update. Otherwise it builds the mask of \Cref{sec:family} once, at the last coarse level $p^\star$,
\begin{equation}
\cA^{\mathrm{GL}}=
\{(I,J):C(x^c_I,y^c_J)<f^c_I+g^c_J+5\eps_{p^\star}\},
\label{eq:geomloss}
\end{equation}
and reuses the resulting KeOps index ranges \citep{charlier2021kernel} for every fine update.

\textbf{Omitted mass.} One FP64 pass over all pairs of the two problems of \Cref{tab:mask-projection} measures the plan mass each mask omits. At FS2's returned potentials on clustered data, the GeomLoss mask, built at its switch to the points, omits more than half the plan mass of 964 of the $2^{20}$ rows and nearly all of it in 12, while FS2's mask omits at most $6.7\times10^{-6}$ of any row and keeps half as many point pairs. This difference includes the stage at which GeomLoss builds its mask, not only its criterion. At GeomLoss's own returned potentials its mask omits $9.8\times10^{-6}$ of the mass, 0.02\% of its residual of 0.054, so that residual is not mainly omitted mass. On heterogeneous data GeomLoss's mask omits at most $7\times10^{-18}$ of any row, and FS2's at most $2.9\times10^{-5}$.

\FloatBarrier
\subsection{Speed}
\label{app:speed}

The grid covers clustered and heterogeneous clouds with a size sweep over $n=m\in\{2^{19},\dots,2^{23}\}$ at $\eps\in\{0.001,0.01\}$ and a regularization sweep over $\eps\in\{0.001,0.002,0.005,0.01,0.02\}$ at $n=m\in\{2^{21},2^{22}\}$. Eight problems are shared, which leaves 32 inputs (\Cref{tab:grid-endpoints}). \Cref{fig:benchmarks}(a--d) shows the size sweep and the $n=2^{22}$ regularization sweep, 26 problems. FS2 runs its automatic policy (\Cref{app:block-size,app:stopping-rules}) with one fixed kernel calibration and no per-problem tuning.

\begin{table}[!htb]
\centering\small
\setlength{\tabcolsep}{5pt}
\caption{The 32-problem forward grid. A dagger marks $R>0.005$. FS2 meets the target on all 32 problems and GeomLoss on 10, all heterogeneous. Each entry is one call. \Cref{tab:clustered-solve}a and \Cref{tab:geomloss-sweep} remeasure clustered $2^{21}$, $\eps=0.001$ as medians of three calls.}
\label{tab:grid-endpoints}
\begin{tabular}{rrrrrr}
\toprule
$n=m$ & $\eps$ & FS2 time (s) & FS2 $R$ & GL time (s) & GL $R$ \\
\midrule
\multicolumn{6}{l}{\textit{Clustered}} \\
$2^{19}$ & 0.001 & 5.638 & 0.004879 & 1.521 & 0.054122$^{\dagger}$ \\
$2^{19}$ & 0.01 & 0.290 & 0.002918 & 0.026 & 0.072730$^{\dagger}$ \\
$2^{20}$ & 0.001 & 4.176 & 0.004791 & 5.500 & 0.053802$^{\dagger}$ \\
$2^{20}$ & 0.01 & 0.275 & 0.003245 & 0.027 & 0.071918$^{\dagger}$ \\
$2^{21}$ & 0.001 & 15.593 & 0.004828 & 21.448 & 0.054161$^{\dagger}$ \\
$2^{21}$ & 0.002 & 5.380 & 0.004094 & 14.305 & 0.055280$^{\dagger}$ \\
$2^{21}$ & 0.005 & 2.318 & 0.003198 & 7.167 & 0.056265$^{\dagger}$ \\
$2^{21}$ & 0.01 & 0.812 & 0.002673 & 0.031 & 0.073064$^{\dagger}$ \\
$2^{21}$ & 0.02 & 0.646 & 0.001066 & 0.029 & 0.046715$^{\dagger}$ \\
$2^{22}$ & 0.001 & 20.289 & 0.004639 & 85.484 & 0.053612$^{\dagger}$ \\
$2^{22}$ & 0.002 & 9.843 & 0.004871 & 56.868 & 0.054748$^{\dagger}$ \\
$2^{22}$ & 0.005 & 3.676 & 0.004129 & 28.455 & 0.055743$^{\dagger}$ \\
$2^{22}$ & 0.01 & 2.825 & 0.001989 & 0.037 & 0.070512$^{\dagger}$ \\
$2^{22}$ & 0.02 & 2.248 & 0.000810 & 0.037 & 0.046228$^{\dagger}$ \\
$2^{23}$ & 0.001 & 45.405 & 0.004694 & 346.051 & 0.054192$^{\dagger}$ \\
$2^{23}$ & 0.01 & 9.914 & 0.001276 & 0.052 & 0.071911$^{\dagger}$ \\
\midrule
\multicolumn{6}{l}{\textit{Heterogeneous}} \\
$2^{19}$ & 0.001 & 0.929 & 0.003605 & 10.547 & 0.012514$^{\dagger}$ \\
$2^{19}$ & 0.01 & 0.045 & 0.001944 & 3.314 & 0.002151 \\
$2^{20}$ & 0.001 & 1.824 & 0.004342 & 39.628 & 0.013687$^{\dagger}$ \\
$2^{20}$ & 0.01 & 0.106 & 0.000970 & 14.886 & 0.000886 \\
$2^{21}$ & 0.001 & 7.060 & 0.003875 & 154.122 & 0.013793$^{\dagger}$ \\
$2^{21}$ & 0.002 & 0.302 & 0.001911 & 125.213 & 0.007890$^{\dagger}$ \\
$2^{21}$ & 0.005 & 0.285 & 0.000962 & 86.684 & 0.002550 \\
$2^{21}$ & 0.01 & 0.272 & 0.000949 & 57.820 & 0.000841 \\
$2^{21}$ & 0.02 & 0.261 & 0.000548 & 19.300 & 0.001679 \\
$2^{22}$ & 0.001 & 1.096 & 0.003059 & 612.415 & 0.013734$^{\dagger}$ \\
$2^{22}$ & 0.002 & 1.053 & 0.001056 & 497.391 & 0.007850$^{\dagger}$ \\
$2^{22}$ & 0.005 & 0.992 & 0.000644 & 344.334 & 0.002531 \\
$2^{22}$ & 0.01 & 0.947 & 0.000880 & 229.573 & 0.000846 \\
$2^{22}$ & 0.02 & 0.908 & 0.000531 & 76.566 & 0.001682 \\
$2^{23}$ & 0.001 & 4.026 & 0.001624 & 2595.417 & 0.003906 \\
$2^{23}$ & 0.01 & 3.484 & 0.000885 & 915.975 & 0.000883 \\
\bottomrule
\end{tabular}

\end{table}

\textbf{Trajectories.} In \Cref{fig:benchmarks}(f), FS2's potentials are lifted every 64 coarse updates and taken after every fine update, and each is placed at the time in the solve at which it is available, its own lift included.

\begin{table}[!htb]
\centering\small
\setlength{\tabcolsep}{4pt}
\caption{GeomLoss settings on two $n{=}m{=}2^{21}$ problems, truncation 5, rerun with FS2 in the same session. Bold marks solves that meet $\tau{=}0.005$. Default is GeomLoss's cluster scale (430 and 2033 source cells), and matched gives as many source cells as FS2's coarse stage, to within one (23,799 against 23,800, and 28,332). FS2 and matched times are medians of three calls, the others single calls. $R$ is the full two-sided residual.}
\label{tab:geomloss-sweep}
\begin{tabular}{@{}llrrrr@{}}
\toprule
 & & \multicolumn{2}{c}{Clustered, $\eps{=}0.001$} & \multicolumn{2}{c}{Heterogeneous, $\eps{=}0.01$} \\
\cmidrule(lr){3-4}\cmidrule(l){5-6}
Solver & Cells & Time (s) & $R$ & Time (s) & $R$ \\
\midrule
FS2 & auto & 15.37 & \textbf{.00483} & 0.273 & \textbf{.00095} \\
\addlinespace
GL, scaling .90 & default & 21.46 & .0542 & 57.80 & \textbf{.00084} \\
GL, scaling .95 & default & 39.48 & .0310 & 86.59 & \textbf{.00094} \\
GL, scaling .99 & default & 200.38 & .00687 & 404.03 & \textbf{.00020} \\
\addlinespace
GL, scaling .90 & matched & 0.361 & .0597 & 0.391 & \textbf{.00076} \\
GL, scaling .95 & matched & 0.548 & .0373 & 0.567 & \textbf{.00083} \\
GL, scaling .99 & matched & 2.023 & .0171 & 2.012 & \textbf{.00036} \\
\bottomrule
\end{tabular}

\end{table}

\textbf{GeomLoss settings.} The grid runs GeomLoss at scaling 0.9 with its default cluster scale and truncation 5. \Cref{tab:geomloss-sweep} varies the scaling and the cluster scale on two $n=m=2^{21}$ problems. The matched cluster scale gives GeomLoss FS2's cell count, taken from FS2's record and not charged to GeomLoss, and at this scale GeomLoss's schedule stays on the centroids and ends with a lift. On clustered data at $\eps=0.001$ none of the six settings meets $\tau$. The best, scaling 0.99 at the default cluster scale, returns 0.0069 after 200\,s, 13 times FS2's time. FS2's own coarse stage also misses $\tau$ there, so the target needs fine updates, and GeomLoss's fine stage in these settings does not reach it. On heterogeneous data at $\eps=0.01$ all six settings meet $\tau$. At the default cluster scale GeomLoss takes 57.8\,s, 212 times FS2, while with FS2's cell count its coarse exit takes 0.39\,s, 1.4 times FS2. The two scaling-0.9 runs differ only in the cluster scale, so on this problem the gap under the default setting comes from the cluster scale.

\begin{table}[!htb]
\centering\small
\setlength{\tabcolsep}{4pt}
\caption{GeomLoss continued at the target $\eps$ on clustered $n{=}m{=}2^{21}$, truncation 5, on the GPU of \Cref{tab:geomloss-sweep}. $K^*$ is the smallest number $K$ of appended target-$\eps$ entries whose full residual meets $\tau{=}0.005$. Time is the median of three clean calls at $K^*$, which at $\eps{=}0.001$ are the GL$^{+}$ calls of \Cref{tab:clustered-solve}a. Bold marks solves that meet $\tau$. Without a $K^*$, the last column gives the lowest scored residual. FS2 takes 15.44\,s with $R{=}0.00483$ at $\eps{=}0.001$ (\Cref{tab:clustered-solve}a) and, in the run of this table, 0.817\,s with $R{=}0.00267$ at $\eps{=}0.01$.}
\label{tab:geomloss-continued}
\begin{tabular}{@{}llrrrr@{}}
\toprule
 & & $K{=}0$ & \multicolumn{3}{c}{Continued at the target $\eps$} \\
\cmidrule(lr){3-3}\cmidrule(l){4-6}
Problem & GL setting & $R$ & $K^*$ & Time (s) & $R$ \\
\midrule
$\eps{=}0.001$ & .90, default & .0542 & -- & -- & .0203 at $K{=}512$ \\
 & .99, default & .00687 & 118 & 421.10 & \textbf{.004997} \\
 & .90, matched & .0597 & -- & -- & .0160 at $K{=}2048$ \\
\addlinespace
$\eps{=}0.01$ & .90, default & .0731 & -- & -- & .0372 at $K{=}512$ \\
 & .99, default & .0373 & -- & -- & .0372 at $K{=}512$ \\
 & .90, matched & .0591 & 183 & 1.13 & \textbf{.004976} \\
\bottomrule
\end{tabular}

\end{table}

\textbf{GeomLoss continued at the target $\eps$.} The grid gives GeomLoss no residual-based stopping rule, so here we give it one that favours it. We append $K$ copies of the target $\eps$ to its native schedule and stop at $K^*$, found by scoring $K=0,1,2,4,\dots,2048$ and bisecting the first bracket that meets $\tau$. This oracle stop charges GeomLoss for no residual test, whereas FS2's times include its sampled checks. On clustered $\eps=0.001$, scaling 0.99 at the default cluster scale reaches $\tau$ after 118 extra entries in 421\,s, 27.3 times FS2's 15.4\,s (\Cref{tab:geomloss-continued}), and the grid's setting reaches it at none of the scored $K$. At $\eps=0.01$ the matched setting reaches $\tau$ in 1.13\,s, 1.38 times FS2, while both default-cell settings level off at 0.037. Both results are specific to these two problems.

\subsection{Residual Certificate and Sampled Check}
\label{app:certificate}

\begin{table}[!htb]
\centering\small
\setlength{\tabcolsep}{4pt}
\caption{The certificate $R_{\cA}+\Omega$ of \Cref{sec:rule} at returned potentials, with $R_{\cA}$ the larger of the row and column residuals of $\bP^{\cA}$ in IEEE FP32 and $\Omega$ the tile-sum bound on the omitted mass in FP64. Where the mask would admit more than half of the tiles (dense), the update is dense, $\Omega=0$ and $R_{\cA}$ is the full residual. $R$ is the recorded full residual, and bold certifies $R\le\tau$.}
\label{tab:certificate}
\begin{tabular}{@{}lrrrlrrrr@{}}
\toprule
Problem & $n$ & $\eps$ & $\tau$ & Mask & $R$ & $R_{\cA}$ & $\Omega$ & $R_{\cA}+\Omega$ \\
\midrule
Clustered & $2^{19}$ & 0.001 & 0.005 & sparse & 0.004879 & 0.004879 & $4.7{\times}10^{-6}$ & \textbf{0.004884} \\
Clustered & $2^{19}$ & 0.01 & 0.005 & sparse & 0.002918 & 0.002918 & $2.7{\times}10^{-5}$ & \textbf{0.002945} \\
Clustered & $2^{20}$ & 0.001 & 0.005 & sparse & 0.004791 & 0.004791 & $4.7{\times}10^{-6}$ & \textbf{0.004796} \\
Clustered & $2^{20}$ & 0.01 & 0.005 & sparse & 0.003245 & 0.003246 & $2.5{\times}10^{-5}$ & \textbf{0.003271} \\
Clustered & $2^{21}$ & 0.001 & 0.005 & sparse & 0.004828 & 0.004828 & $4.8{\times}10^{-6}$ & \textbf{0.004833} \\
Clustered & $2^{21}$ & 0.002 & 0.005 & sparse & 0.004094 & 0.004094 & $7.5{\times}10^{-6}$ & \textbf{0.004101} \\
Clustered & $2^{21}$ & 0.005 & 0.005 & sparse & 0.003198 & 0.003198 & $1.4{\times}10^{-5}$ & \textbf{0.003211} \\
Clustered & $2^{21}$ & 0.01 & 0.005 & sparse & 0.002673 & 0.002674 & $2.3{\times}10^{-5}$ & \textbf{0.002697} \\
Clustered & $2^{21}$ & 0.02 & 0.005 & sparse & 0.001066 & 0.001067 & $4.9{\times}10^{-5}$ & \textbf{0.001115} \\
Clustered & $2^{22}$ & 0.001 & 0.005 & sparse & 0.004639 & 0.004639 & $4.8{\times}10^{-6}$ & \textbf{0.004644} \\
\addlinespace
Heterog. & $2^{19}$ & 0.001 & 0.005 & dense & 0.003605 & 0.003605 & 0 & \textbf{0.003605} \\
Heterog. & $2^{20}$ & 0.001 & 0.005 & dense & 0.004342 & 0.004342 & 0 & \textbf{0.004342} \\
Heterog. & $2^{21}$ & 0.001 & 0.005 & dense & 0.003875 & 0.003875 & 0 & \textbf{0.003875} \\
\addlinespace
Quijote & $2^{27}$ & 3.81 & 0.01 & sparse & 0.009154 & 0.009108 & $1.1{\times}10^{-8}$ & \textbf{0.009108} \\
Quijote & $2^{27}$ & 3.81 & 0.05 & sparse & 0.048965 & 0.048953 & $5.4{\times}10^{-8}$ & \textbf{0.048953} \\
\bottomrule
\end{tabular}

\end{table}

\textbf{Certificate.} The solver as run stops on its sampled check, and the certificate is computed afterwards. At each evaluated return we build the mask by the production rule at the returned pair and compute $R_{\cA}$ with the solver's sparse kernel. We evaluated the clustered returns up to $n=2^{22}$ at $\eps=0.001$ and up to $2^{21}$ at the other regularizations, the heterogeneous returns at $\eps=0.001$ up to $2^{21}$, and both $2^{27}$ Quijote solutions, but not the other grid returns. All 15 are certified (\Cref{tab:certificate}). On the grid, $\Omega$ is at most 1\% of $\tau$ and $R_{\cA}$ agrees with the recorded full residual to within $1.2\times10^{-6}$, at the cost of one mask build and one sparse update. The heterogeneous returns take the dense update, so their certificate is the full residual. At $2^{27}$, $\Omega$ is below $10^{-7}$, so FP32 arithmetic rather than the mask sets the certificate, which gives 0.009108 at $\tau=0.01$ against a recorded full residual of 0.009154. Both are FP32 estimates, and the FP64 correction of \Cref{app:quijote-precision} gives .00882 and .00869 for rows and columns. The certificate takes minutes on one A100, where the full evaluation takes some 38 hours.

\textbf{Sampled check.} Of the 40 reported returns with a full residual, none exceeds $\tau$, and the largest is $0.98\tau$. We re-solved the seven grid problems closest to $\tau$, clustered $\eps=0.001$ at $n=2^{19}$, $2^{20}$ and $2^{21}$, clustered $2^{21}$ at $\eps=0.002$, and heterogeneous $\eps=0.001$ at $2^{19}$, $2^{20}$ and $2^{21}$, with the check's index seeds set to $(s,159015+s)$ for $s=1,\dots,20$. All 140 solves returned confirmed pairs, and none has a full residual above $\tau$, which bounds the rate of false acceptance at 2.6\% (Clopper--Pearson, 95\%), or at 6.0\% over the 60 solves on the three problems whose stopping update moved with the seed. At the accepting check, $(R-\text{mean})/\mathrm{SE}$ over the 280 sides has mean 0.03 and standard deviation 1.01, so reusing one sample across checks leaves no detectable bias at the stop.

\textbf{Concentrated error.} On the accepted pair of heterogeneous $2^{20}$ at $\eps=0.01$ ($R=0.19\tau$), we shift $f$ on a random 0.01\%, 0.1\% or 1\% of the rows so that the row residual becomes $1.2\tau$ or $2\tau$, and run the check with 1000 index seeds per setting. The row side alone accepts 95.8\%, 38.0--38.9\% and 0.1--0.6\% of the seeds, close to the rates predicted from the exact per-row values. The production check also tests the columns, whose residual the shift raises to 1.15--1.96$\tau$, and accepted none of the 6000 perturbed states. These measurements do not extend to $n\ge2^{22}$, to other $\eps$, to non-uniform weights or to the Quijote blur.

\subsection{Ablations}
\label{app:ablations}

The ablations use clustered and heterogeneous clouds with $n=m=2^{20}$ at $\eps\in\{0.001,0.01\}$, and pair each control with an FS2 reference measured in the same run on the same GPU. Each of three controls removes one component. \emph{No coarse stage} anneals on the points instead of the centroids, with no lift and no coarse continuation. \emph{Dense fine} keeps the coarse stage and replaces every fine update by a dense streaming update. \emph{Box only} keeps the box screen and sparse execution but skips tile verification.

\begin{table}[!htb]
\centering\footnotesize
\setlength{\tabcolsep}{3pt}
\caption{Component ablations. C and H denote clustered and heterogeneous inputs, followed by $\eps$. The ratio (control/FS2) is given for active controls that meet the target.}
\label{tab:ablations}
\begin{tabular}{llrrrrl}
\toprule
Case & Control & FS2 (s) & Control (s) & Full $R$ & Ratio & State \\
\midrule
C, 0.001 & No coarse stage & 4.165 & 153.330 & 0.02135 & -- & misses target \\
 & Dense fine & 4.165 & 33.327 & 0.00479 & 8.001 & qualified \\
 & Box only & 4.167 & 4.335 & 0.00479 & 1.040 & qualified \\
\addlinespace
C, 0.01 & No coarse stage & 0.275 & 183.132 & 0.00477 & 664.818 & qualified \\
 & Dense fine & 0.276 & 0.275 & 0.00325 & -- & inactive \\
 & Box only & 0.276 & 0.275 & 0.00325 & -- & inactive \\
\addlinespace
H, 0.001 & No coarse stage & 1.825 & 281.133 & 0.00039 & 154.023 & qualified \\
 & Dense fine & 1.824 & 2.423 & 0.00434 & 1.328 & qualified \\
 & Box only & 1.824 & 1.576 & 0.00434 & 0.864 & qualified \\
\addlinespace
H, 0.01 & No coarse stage & 0.089 & 145.599 & 0.00318 & 1643.418 & qualified \\
 & Dense fine & 0.089 & 0.089 & 0.00097 & -- & inactive \\
 & Box only & 0.089 & 0.107 & 0.00097 & -- & inactive \\
\bottomrule
\end{tabular}

\end{table}

\textbf{Coarse stage.} Removing the coarse stage slows the solve by 154--1643$\times$ on the three problems where both arms meet the target (\Cref{tab:ablations}). On clustered $\eps=0.001$ the control stops at the 128-update fine cap with $R=0.02135$. A shorter coarse continuation has the same effect on a smaller scale (\Cref{tab:coarse-controls}a): after 64 or 256 extra coarse updates the solve misses $\tau$ within the fine cap, and with the cap raised to 8192 it needs 1305 or 1115 fine updates, 51.6--60.0$\times$ the FS2 time.

\begin{table}[!htb]
\centering\footnotesize
\setlength{\tabcolsep}{4pt}
\caption{Coarse-continuation and threshold controls at $n=m=2^{20}$, $\eps=0.001$. Fine updates count ordinary updates plus terminal previews. (a) Clustered, with the number of extra coarse updates fixed. Ratios divide by the FS2 reference of the same run. (b) Every threshold control meets the target.}
\label{tab:coarse-controls}
\begin{tabular}{@{}rrrrrr@{}}
\toprule
\multicolumn{6}{@{}l}{\textbf{(a)} Extra coarse updates, clustered} \\
Extra coarse & Fine cap & Fine updates & Time (s) & $R$ & Ratio \\
\midrule
2048 (native) & 128 & 13\,+\,2 & 4.169 & 0.00479 & 1 \\
64 & 128 & 128\,+\,3 & 24.355 & 0.01896 & -- \\
256 & 128 & 128\,+\,3 & 24.817 & 0.01421 & -- \\
64 & 8192 & 1267\,+\,38 & 250.235 & 0.00478 & 60.0 \\
256 & 8192 & 1077\,+\,38 & 215.674 & 0.00477 & 51.6 \\
\midrule
\multicolumn{6}{@{}l}{\textbf{(b)} $\delta_{\mathrm{rel}}$ scaled by a factor; time relative to the native threshold} \\
Factor & \multicolumn{2}{r}{Clustered} & \multicolumn{2}{r}{Heterogeneous} & \\
\midrule
$10^{-2}$ & \multicolumn{2}{r}{1.080} & \multicolumn{2}{r}{1.039} & \\
$10^{-1}$ & \multicolumn{2}{r}{1.068} & \multicolumn{2}{r}{1.027} & \\
$10^{1}$ & \multicolumn{2}{r}{0.965} & \multicolumn{2}{r}{0.945} & \\
$10^{2}$ & \multicolumn{2}{r}{0.908} & \multicolumn{2}{r}{1.530} & \\
\bottomrule
\end{tabular}

\end{table}

\textbf{Sparse fine execution.} On clustered $\eps=0.001$ both arms run 13 ordinary updates and two terminal previews, and sparse execution cuts the time from 33.327 to 4.165\,s ($8.00\times$). Heterogeneous $\eps=0.001$ takes one terminal preview over a mask that admits 60.6\% of the tiles and gains $1.33\times$. At $\eps=0.01$ neither geometry takes a fine update.

\textbf{Tile verification.} On clustered $\eps=0.001$, verification cuts the padded point-pair work from $3.049\times10^{12}$ to $2.087\times10^{12}$ and the time from 4.335 to 4.167\,s. On heterogeneous data it removes 3.95\% of the work and costs time, 1.824 against 1.576\,s. Scaling $\delta_{\mathrm{rel}}$ by $10^{-2}$ to $10^{2}$ changes the clustered time by at most 10\%, and the largest factor slows the heterogeneous solve by 53\% (\Cref{tab:coarse-controls}b).

\subsection{Profiling}
\label{app:profiling}

The profiles use the clustered $n=m=2^{21}$, $\eps=0.001$ problem of \Cref{tab:clustered-solve}. The stage times of \Cref{tab:clustered-solve}a come from one Nsight Systems call per solver. FS2's coarse stage is its coarse kernels (2.663\,s) and lifts (0.277\,s), and its fine stage its fine kernels (8.274\,s), mask construction (3.118\,s) and sampled checks (0.909\,s). GeomLoss's fine phase is 96 kernels averaging 0.22\,s, so kernel launches do not explain its time. The kernel counters come from NCU on a separate A100, where the unprofiled solve takes 16.23\,s with FS2 and 21.49\,s with GeomLoss.

\begin{table}[!htb]
\centering\footnotesize
\setlength{\tabcolsep}{4pt}
\caption{Kernel counters for one fine update per direction on clustered $n=m=2^{21}$, $\eps=0.001$, on the profiling A100. FS replays one dense reduction per direction from its native incoming state, with its counters averaged over the two directions. GL and FS2 run FS2's first ordinary fine update on the same $B=128$ mask, and $f/g$ gives the two directions. Per-pair counts divide totals by the point pairs they cover, instructions are executed warp instructions, and dashes mark counters not collected.}
\label{tab:kernel-counters}
\begin{tabular}{lrrr}
\toprule
Metric & FS (dense) & GL (KeOps, FP32) & FS2 (TF32) \\
\midrule
Point pairs per direction & $4.40\times10^{12}$ & $2.58\times10^{11}$ & $2.58\times10^{11}$ \\
NCU time per direction (ms) & 3{,}600 & 393.957 & 251.435 \\
Warm update, both directions (ms) & -- & 791.073 & 504.392 \\
Warp instructions / pair & 0.34 & 0.628315 & 0.278530 \\
Shared-load wavefronts / pair & -- & 0.0625001 & 0.0246587 \\
Shared-store wavefronts / pair & -- & 0.0009806 & 0.0053709 \\
HBM read bytes / pair & 0.0043542 & 0.0003161 & 0.0149489 \\
HBM write bytes / pair & 0.0000025 & 0.0001135 & 0.0000558 \\
Tensor-pipe instructions / pair & -- & 0.0003344 & 0.015625 \\
Tensor activity, $f/g$ (\%) & -- & 0.07 / 0.07 & 21.11 / 21.08 \\
Occupancy, $f/g$ (\%) & 37.08 & 91.4 / 91.14 & 12.48 / 12.47 \\
Math-pipe throttle, $f/g$ (\%) & 8.995 & 17.39 / 17.41 & 2.32 / 2.31 \\
Long scoreboard, $f/g$ (\%) & 1.21 & 0.51 / 0.51 & 6.05 / 6.03 \\
\bottomrule
\end{tabular}

\end{table}

\textbf{Kernel counters.} \Cref{tab:kernel-counters} compares the GL and FS2 executors on the same mask. We freeze FS2's first ordinary fine update on the clustered $n=m=2^{21}$, $\eps=0.001$ problem, with a $B=128$ mask of 15,760,133 admitted block pairs and $5.16\times10^{11}$ point pairs over both directions. FS2's TF32 kernels and GeomLoss's FP32 KeOps executor receive the same sorted points, incoming potentials, weights and transposed support, KeOps with each run of contiguous target blocks as one range. Over 20 warmed repetitions the update takes 504.392\,ms with FS2 and 791.073\,ms with KeOps, $1.568\times$ longer. Per pair, FS2 executes 55.7\% fewer warp instructions and 60.5\% fewer shared-load wavefronts, with tensor-pipe activity of 21\% against 0.07\%, despite more HBM reads, lower occupancy and more long-scoreboard stall. The extra HBM traffic follows from the policy's block size $B=128$ (\Cref{app:block-size}). The dense FlashSinkhorn reduction evaluates all $4.40\times10^{12}$ pairs of one direction in 3.6\,s, $14\times$ longer than FS2's reduction on its mask.

\subsection{Memory}
\label{app:memory}

\Cref{fig:benchmarks}(e) reports the PyTorch allocator peak, inputs included. The same protocol on the five heterogeneous FS2 solves at $\eps=0.001$ gives peaks of 0.80, 2.72, 9.89, 1.70 and 3.40\,GiB for $n=2^{19}$ to $2^{23}$. Up to $2^{21}$, FS2 takes one or two fine updates over 60--61\% of the tiles. From $2^{22}$ it returns the lifted coarse potentials without a fine update and builds no mask.

\textbf{Block size.} We hold the clustered $n=m=2^{23}$, $\eps=0.001$ problem and its coarse partition fixed ($t=42$) and change only the block size. Raising $B$ from 128 to 256 cuts the admitted tiles from $2.25\times10^{8}$ to $6.0\times10^{7}$ and the PyTorch peak from 16.17 to 6.88\,GiB, for 7\% more time, 48.52 against 45.38\,s, with the same returned $R=0.004694$. The automatic policy chooses $B$ from the L2 budget of the block summaries, not from this mask memory.

\FloatBarrier
\subsection{Optimization}
\label{app:optimization}

The optimization grid reuses the 32 forward inputs of \Cref{app:speed}. Each method runs one trajectory of twenty normalized updates (\Cref{sec:optimization}) and saves steps 0, 5, 10, 15 and 20. The timer includes compilation and autotuning, and a run stops at 7200\,s. FS2 steps along the surrogate gradient of \Cref{sec:coarse-direction}, while GeomLoss differentiates its native multiscale loss with autograd. GeomLoss reaches the time limit on six problems: clustered $2^{23}$ at $\eps=0.001$, and heterogeneous $2^{22}$ at $\eps\in\{0.001,0.002,0.005\}$ and $2^{23}$ at $\eps\in\{0.001,0.01\}$. \Cref{tab:optimization} summarizes the $n=m=2^{21}$ runs of \Cref{fig:benchmarks}(g,h).

\textbf{Score.} The evaluator fixes $K$ nested groups of the initial source cloud and of $\bY$, with FP64 centroids and within-group variances. At a saved iterate, $U_K$ is the entropic cost between the source and target group centroids, solved to row and column residuals of at most $10^{-4}$, plus the within-group squared-distance averages. We report $K=131072$ at $\eps\le0.002$, $65536$ at $\eps=0.005$ and $16384$ at $\eps\ge0.01$, and require every decrease to agree with the one at a coarser count ($65536$, $16384$ and $4096$) within $\max(10^{-6},0.02|\Delta U_K|)$. The score ranks trajectories and is not the full-point OT value.

\textbf{Direction at $\bX_0$.} At the initial points of the four $n=m=2^{21}$ runs, we compare $\widetilde G$ and GeomLoss's autograd direction with $G_i=2a_i(x_i-T_\eps(x_i))$ from a separate solve at $\tau=5\times10^{-4}$. On clustered data ($\eps=0.001$ and $0.01$), $\widetilde G$ has cosine 0.92 and 0.96 with $G$ and relative error 0.39 and 0.28, against 0.49 and 0.40 for GeomLoss's direction. On heterogeneous data its cosine is at least 0.9997 and, from 4096 rows evaluated in FP64, its relative error is 0.023 and 0.012, and the difference between the two directions lies within the uncertainty of the references. These values measure alignment with a numerical reference, not the error to the exact gradient, because references at $\tau=0.005$ and $5\times10^{-4}$ differ by 0.2--3\%, and on clustered data at $\eps=0.001$ the tighter solve stopped at residual $1.2\times10^{-3}$, above its target. On the same 4096 rows, on three of the four problems, the right-hand side of \Cref{eq:coarse-gradient-error} exceeds $\|G\|_F$, so \Cref{prop:coarse-barycentric} does not certify a relative error below one there.

\begin{table}[!htb]
\centering
\small
\setlength{\tabcolsep}{4pt}
\caption{Twenty-step optimization at $n=m=2^{21}$. A decrease is the percentage drop of $U$ from the common initial score $U_0$ over the twenty steps, and final scores within $3\times10^{-6}$ tie. Times cover the twenty updates, including compilation and autotuning.}
\label{tab:optimization}
\begin{tabular}{@{}llrrrrr@{}}
\toprule
Cloud & $\eps$ & $U_0$ & FS2 s & GL s & FS2 decrease & GL decrease \\
\midrule
Heterogeneous & .001 & .004297 & 52.80 & 3215.67 & .5829\% & .5826\% \\
Heterogeneous & .01  & .01501 & 51.66 & 1290.01 & .2551\% & .2550\% \\
Clustered     & .001 & .05230 & 79.81 & 453.22  & .3064\% & .2906\% \\
Clustered     & .01  & .08166 & 42.58 & 0.80    & .1756\% & .1682\% \\
\bottomrule
\end{tabular}

\end{table}

\FloatBarrier
\section{Quijote Particle Transport}
\label{app:quijote}

\subsection{Problem and Data}

The sources are redshift-zero positions matched by particle ID to their Lagrangian lattice sites, the targets. The $2^{24}$ instance is the central $256^3$ subcube and the $2^{27}$ instance the full $512^3$ box, with $\eps=3.814697265625$, so that the blur $\sqrt\eps=1.953125\,\mathrm{Mpc}/h$ equals the lattice spacing, and a nonperiodic cost. FS2 solves on TF32-rounded coordinates, which move each particle by at most 0.125\,Mpc/$h$, 6\% of the blur. In the $2^{24}$ subcube the particle-ID displacement $\|x_i-q_i\|$ has median 9.013, 99th percentile 20.668 and maximum 39.135\,Mpc/$h$.

\FloatBarrier
\subsection{Configuration and Timing}
\label{app:quijote-config}

\begin{table}[!htb]
\caption{Quijote forward solves, with the $\tau=0.05$ size sweep in the first four rows. $B/t/c/h$: block size, cell level, fine cap and returned ordinary updates. Peak: PyTorch allocator peak of the size sweep. Light validation: mean absolute relative marginal error on 4096 sampled indices per direction, 65,536 at $2^{25}$ and $2^{26}$, with the mean plus three standard errors in brackets. Full max: $R$, with dashes where it is not enumerated. W: warm call, F: first call, I: instrumented call. MS/.95/default/5: multiscale, scaling 0.95, default cluster scale, truncation 5.}
\label{tab:quijote}
\centering
\footnotesize
\setlength{\tabcolsep}{2.5pt}
\begin{tabular}{@{}llrlrrrr@{}}
\toprule
method & $n=m$ & $\tau$ & configuration & time (s) & peak (GiB) & light validation $[U_{3\mathrm{SE}}]$ & full max \\
\midrule
FS2 & $2^{24}$ & .05 & $B64/t47/c128/h80$ & 146.905$^{\mathrm W}$ & 12.053 & .046836 [.048531] & .046631\,\cmark \\
FS2 & $2^{25}$ & .05 & $B64/t45/c128/h48$ & 374.839$^{\mathrm W}$ & 19.472 & .049304 [.049758] & \textemdash \\
FS2 & $2^{26}$ & .05 & $B128/t45/c128/h72$ & 694.630$^{\mathrm W}$ & 33.131 & .048340 [.048783] & \textemdash \\
FS2 & $2^{27}$ & .05 & $B256/t45/c128/h80$ & 1351.589$^{\mathrm W}$ & 64.129 & .049009 [.050797] & .048965\,\cmark \\
\addlinespace
FS2 & $2^{24}$ & .01 & $B64/t47/c800/h432$ & 736.393$^{\mathrm W}$ & \textemdash & .009538 [.009924] & .009433\,\cmark \\
FS2 & $2^{27}$ & .01 & $B256/t45/c800/h560$ & 8832.994$^{\mathrm F}$ & \textemdash & .009585 [.009954] & .009154\,\cmark \\
\addlinespace
native GL & $2^{24}$ & .05 & MS/.95/default/5 & 3713.300$^{\mathrm I}$ & \textemdash & .058976 [.060909] & .058463\,\xmark \\
\bottomrule
\end{tabular}

\end{table}

\textbf{Configuration.} At $2^{24}$, FS2 uses $t=47$, the selector's $B=64$ and a sampled check every 16 fine updates. At $2^{27}$ it uses the selector's $B=256$, processed as $64\times64$ point tiles, $t=45$ and an automatic check interval. The fine cap is 128 at $\tau=0.05$ and 800 at $\tau=0.01$, where the two sizes return after 432 and 560 ordinary updates. At every size the solver completes the 64 entries at $\eps$ before its first lift. In place of the continuation of Appendix~\ref{app:stopping-rules}, it then runs at most one batch of 64 to 384 coarse updates after the second fine update and lifts again, and in every run the check scored that lift worse than the fine iterate. The time of this batch is included. We set $t$ by hand. At $2^{27}$ the rule would select $t=46$ if its fourth condition admitted that size, but with $t=46$ the sampled check became non-finite in the first fine updates, so we use $t=45$, whose coarse state exceeds the L2 budget of \Cref{eq:automatic-coarse-selection}. With the continuation of Appendix~\ref{app:stopping-rules} instead of the single batch, the $2^{27}$, $\tau=0.05$ solve took 2030.4\,s and 119 ordinary updates, against 1351.6\,s and 80 in \Cref{tab:quijote}. GeomLoss runs only at $2^{24}$ and $\tau=0.05$, at scaling 0.95, default cluster scale and truncation 5.

\textbf{Mask size.} At $2^{27}$ and $\tau=0.01$, in the update that produced the returned potentials, the box screen passed 107,397,067 of the $2^{38}$ block pairs and verification admitted 24,451,822, which hold $1.60\times10^{12}$ point pairs. A truncated kernel with one FP32 value and one 32-bit index per pair would take 11.7\,TiB. FS2's index for the candidate and admitted tiles, at 24 bytes per tile, takes 2.9\,GiB.

\textbf{Timed calls.} Every time in \Cref{tab:quijote} is a single call. FS2's are warm calls, except the first complete call at $2^{27}$ and $\tau=0.01$, and at $2^{24}$ a second call returns identical potentials, so the all-particle validation applies to the timed call. GeomLoss's is the instrumented call whose potentials are scored. Validation is not timed. The light validation of \Cref{tab:quijote} uses seed 1 and reduces each sampled marginal against every point of the opposite measure. At $2^{27}$ the IEEE-FP32 evaluator of the full residual takes some 38 hours on an A100 for the $\tau=0.01$ solution.

\subsection{Evaluator Precision}
\label{app:quijote-precision}

The full residuals come from an IEEE-FP32 evaluator that forms each row's log marginal ratio as the difference of two potentials shifted by $\|x_i\|^2$. Near the box corners $\|x_i\|^2$ reaches $7.5\times10^5$, whose FP32 spacing is 1.6\% of $\eps$, so a single row's mass is resolved only to 0.1--1.6\%. On 4096 or 16384 sampled rows and columns per entry, we recompute each row's error in FP64 from coordinate differences and subtract the mean paired difference from the full FP32 residual. This gives .04653/.04661, .04883/.04889, .00939/.00942 and .00882/.00869 (row/column) for the rows of \Cref{tab:quijote} with a full residual, in table order, with three standard errors of at most .0002, so every entry stays below its $\tau$.

\subsection{Size Sweep}
\label{app:quijote-size-sweep}

\Cref{fig:quijote-scaling}(a,b) and the first four rows of \Cref{tab:quijote} fix $\tau=0.05$ and use FS2 only, on nested cuboids centered in the box from $256^3$ to $512^3$ lattice sites with the lattice spacing and $\eps$ fixed. At $2^{24}$ the time and the memory come from two calls with identical potentials. The middle sizes are single warm calls with the $2^{27}$ configuration and the selector's $B$. Both pass the solver's sampled check and are validated on 65,536 sampled outputs per side, without enumerating the full residual.

\subsection{Displacement Reconstruction}
\label{app:quijote-reconstruction}

We take the $2^{27}$-particle, $\tau=0.01$ solution of \Cref{tab:quijote} and a uniform sample of 4096 source indices, drawn without replacement with seed 20260919. Each projection sums over all $2^{27}$ targets without a mask, using the plan of the TF32-rounded input and the unrounded physical target positions. With $L=1000\,\mathrm{Mpc}/h$ and the componentwise minimum image $w_L(v)=v-L\lfloor v/L+1/2\rfloor$, we evaluate
\[
 T_i=\frac{\sum_jP_{ij}q_j^{\mathrm{phys}}}{\sum_jP_{ij}},\qquad
 \hat\psi_i=w_L(x_i^{\mathrm{phys}}-T_i),\qquad
 \psi_i^*=w_L(x_i^{\mathrm{phys}}-q_i^{\mathrm{phys}}).
\]
The wrapping enters only the scoring, and the solver's cost stays nonperiodic. Over the sample $\mathcal S$, displacement RMSE is $[|\mathcal S|^{-1}\sum_{i\in\mathcal S}\|\hat\psi_i-\psi_i^*\|^2]^{1/2}$, and normalized RMSE divides it by the RMS reference displacement (\Cref{tab:quijote-reconstruction}). The component correlations are 0.797, 0.789 and 0.802, and the unwrapped position RMSE is 122.268\,Mpc/$h$.

\begin{table}[!htb]
\caption{Displacement reconstruction from the $2^{27}$-particle FS2 solution at $\tau=0.01$ on 4096 sampled particles, with minimum-image scoring. The interval comes from 500 bootstrap resamples over particles at 95\% confidence. The subsets split the sample by whether the reference trajectory crosses a periodic face.}
\label{tab:quijote-reconstruction}
\centering\small
\begin{tabular}{@{}lrrrr@{}}
\toprule
Predictor & Particles & RMSE (Mpc/$h$) & Normalized RMSE & 95\% interval \\
\midrule
Zero displacement & 4096 & 10.064 & 1.000 & -- \\
FS2 barycentric displacement & 4096 & 6.473 & 0.643 & [0.625, 0.662] \\
\quad no periodic-face crossing & 4034 & 6.333 & -- & -- \\
\quad periodic-face crossing & 62 & 12.603 & -- & -- \\
\bottomrule
\end{tabular}

\end{table}

The 4096-query projection takes 9.303\,s on one A100, excluding 43.474\,s of loading and preparation, and FP64 direct-distance projections of eight sampled queries agree with the FP32 stream within 0.002281\,Mpc/$h$.

\end{document}